\documentclass[twoside]{article}

\usepackage[accepted]{aistats2026}
\usepackage[utf8]{inputenc} % allow utf-8 input
\usepackage[T1]{fontenc} % use 8-bit T1 fonts
\usepackage{nicefrac} % compact symbols for 1/2, etc.

\usepackage{enumitem}
\setlist[itemize]{noitemsep, topsep=0pt}
\setlist[enumerate]{noitemsep, topsep=0pt}

\usepackage[round]{natbib}

\usepackage{titletoc}
\usepackage[page,header]{appendix}

\usepackage{url} % simple URL typesetting
\usepackage{hyperref}
\hypersetup{
    colorlinks,
    citecolor=blue,
    filecolor=blue,
    linkcolor=blue,
    urlcolor=blue,
}

\usepackage{graphicx}
\usepackage{chngcntr} % make figure numbering in appendix work
\begingroup
    \expandafter\ifx\csname pdfsuppresswarningpagegroup\endcsname\relax\else\global\pdfsuppresswarningpagegroup=1\relax\fi
\endgroup
\usepackage[inkscapelatex=false]{svg} %for svg images
\usepackage{tikz}

\usepackage{makecell}
\usepackage{tabularx}
\usepackage{booktabs}
\usepackage{multirow}
\usepackage{colortbl}
\usepackage{tablefootnote}
\usepackage{array}

\usepackage{algpseudocode}
\usepackage{amsmath,amssymb,amsfonts}
\usepackage{amsthm}

\usepackage{caption}
\usepackage{subcaption}
\newtheorem{thm}{Theorem}
\newtheorem{lem}[thm]{Lemma}
\newtheorem{defn}[thm]{Definition}

\usepackage{xcolor}

\newcommand*\circled[1]{\tikz[baseline=(char.base)]{
            \node[shape=circle,draw,inner sep=2pt] (char) {#1};}}

\begin{document}

%\runningtitle{The Partial Volume Over the ROC Surface}
\runningtitle{Partial VOROS}

\twocolumn[

\aistatstitle{Partial VOROS: A Cost-aware Performance Metric for Binary Classifiers with Precision and Capacity Constraints}

\aistatsauthor{ Christopher Ratigan\footnotemark[1], Kyle Heuton\footnotemark[2], Carissa Wang\footnotemark[2], Lenore Cowen\footnotemark[1]\footnotemark[2], Michael C. Hughes\footnotemark[2]}

\aistatsaddress{\footnotemark[1]Department of Mathematics, Tufts University, Medford, MA, USA\\\footnotemark[2]Department of  Computer Science, Tufts University, Medford, MA, USA
}%endaddress 

]

\begin{abstract}
The ROC curve is widely used to assess binary classifiers. 
Yet for some applications, such as alert systems for monitoring hospitalized patients, conventional ROC analysis cannot meet two key deployment needs:  enforcing a constraint on precision to avoid false alarm fatigue and imposing an upper bound on the number of predicted positives to represent the capacity of hospital staff.
The usual area under the curve metric also does not reflect asymmetric costs for false positives and false negatives. 
In this paper we address all three of these issues.
First, we show how the subset of classifiers that meet precision and capacity constraints occupy a feasible region in ROC space.
We establish the polygon-shaped geometry of this region. 
We then define the \emph{partial area of lesser classifiers}, a performance metric that is monotonic with cost and only accounts for the feasible region.
% of ROC space. 
Averaging this area over a desired distribution for cost parameters results in the partial volume over the ROC surface, or \emph{partial VOROS}. 
In experiments predicting mortality risk from vital sign history on several datasets, we show this cost-aware metric can outperform alternatives at ranking classifiers for in-hospital alerts.
\end{abstract}

\section{INTRODUCTION}

When a classifier of binary events is deployed in a high-stakes application, it must respect important operational constraints. 
First, context-specific costs mean that false positive predictions usually have different consequences than false negatives in a given task. 
Developing and evaluating classifiers in a cost-aware way is key to deployment success~\citep{provost1997analysis,drummond2000explicitly}.
Second, stakeholders may specify viable ranges for certain performance metrics for the system to be beneficial.
Finally, stakeholders may have capacity constraints, in terms of the overall number of positive predictions or negative predictions they can handle smoothly when deployed.

In this work, our goal is to develop a  performance metric that can effectively rank classifiers when costs, performance constraints, and capacity constraints all matter.
Previous work has suggested many metrics and visuals for evaluating binary classifiers, surveyed later in Sec.~\ref{sec:related_work}.
We take as a starting point a cost-aware analysis of the receiver-operating-characteristic (ROC) curve.
Recent work by \citet{ratigan2024voros}   lifts the classic 2D ROC curve into a 3D surface where the third axis defines the cost. They proposed a performance metric, the volume over the ROC surface or \emph{VOROS}, that can identify when a binary classifier outperforms another given a task-specific range of estimated costs.
Our work here extends this analysis to incorporate constraints on precision and capacity, which are critical in many applications such as fraud detection, credit scoring, and early warning systems.
\let\thefootnote\relax\footnotetext{\noindent Code: \href{https://github.com/tufts-ml/partial-VOROS}{github.com/tufts-ml/partial-VOROS}.}

As a motivating task of interest, consider the evaluation of alert systems for monitoring the health of hospitalized patients.
Here, a classifier must take in recent data about a patient's health and determine whether or not to alarm. 
An alarm indicates the patient's health may be deteriorating, so doctors or nurses should check on that patient soon.
%Alerts are possible at defined times throughout a patient's stay.
An ever-increasing body of literature has developed such systems via machine learning~\citep{abellaalvarezICUWallsProject2013,hylandEarlyPredictionCirculatory2020,sendakRealWorldIntegrationSepsis2020,muralitharanMachineLearningBased2021,edelsonEarlyWarningScores2024}.
Careful evaluation of such systems is critical to ensure they balance tradeoffs appropriately and provide a net benefit to the patient population and the hospital.

Hospital-based alert systems naturally have the three aforementioned operational constraints:

\begin{itemize}[leftmargin=*]
    \item First, there are asymmetric costs to the different kinds of mistakes. A false positive has some cost by taking valuable time from clinical staff that could be spent on other patients with greater needs. A false negative incurs even more cost, as this means a patient did get sicker or even die and an alert that might have helped was never issued.
    \item Second, hospital staff often express a key performance metric constraint: the alert classifier's \emph{precision}, the fraction of all alarms that are true positives, needs to meet some minimum value for the system to be viable~\citep{harrisonAutomatedSepsisDetection2015,rath2022optimizing}. 
        A survey of critical care physicians in South Korean hospitals~\citep{parkCurrentStatusRapid2022} found that too many false positives were the top concern about early warning systems; the median response requested a precision of at least 28.5\%. % = 1/(1+2.5)
    Clinical staff may learn to ignore many or all alarms from the system altogether if its precision is not tolerable, a problem known as \emph{alarm fatigue}~\citep{cvachMonitorAlarmFatigue2012}. Ignoring alarms entirely due to low precision has been a documented safety concern for decades~\citep{sendelbachAlarmFatiguePatient2013,albanowskiTenYearsLater2023}.
    \item Finally, the staff's capacity to respond to alarms must be accounted for. In a typical U.S. ICU, the patient-to-doctor ratio is on average 11.8 and almost always above 6~\citep{kahnIntensivistPhysiciantopatientRatios2023}. If alarms for all patients happened at once, only a fraction could be triaged right away. It is crucial to ensure system evaluation takes into account such constraints.
\end{itemize}

In this paper, we develop a new performance metric, the \emph{Partial VOROS}, which extends the VOROS from \citet{ratigan2024voros} to account for precision and capacity constraints while preserving the original VOROS' ability to handle asymmetric costs.
In Sec.~3, we formalize how these constraints narrow the entirety of ROC space to a feasible region satisfying all constraints.
In Sec.~4, we provide formulas for evaluating areas of lesser classifiers and volumes over this feasible region using desired cost distributions.
In Sec.~5 and 6, we examine prediction tasks on real health records data, showing how our partial VOROS can identify promising classifiers by accounting for constraints and costs in ways other metrics cannot.

\begin{figure*}[!t]
    \begin{tabular}{c c c}    \includegraphics[width=0.4\textwidth]{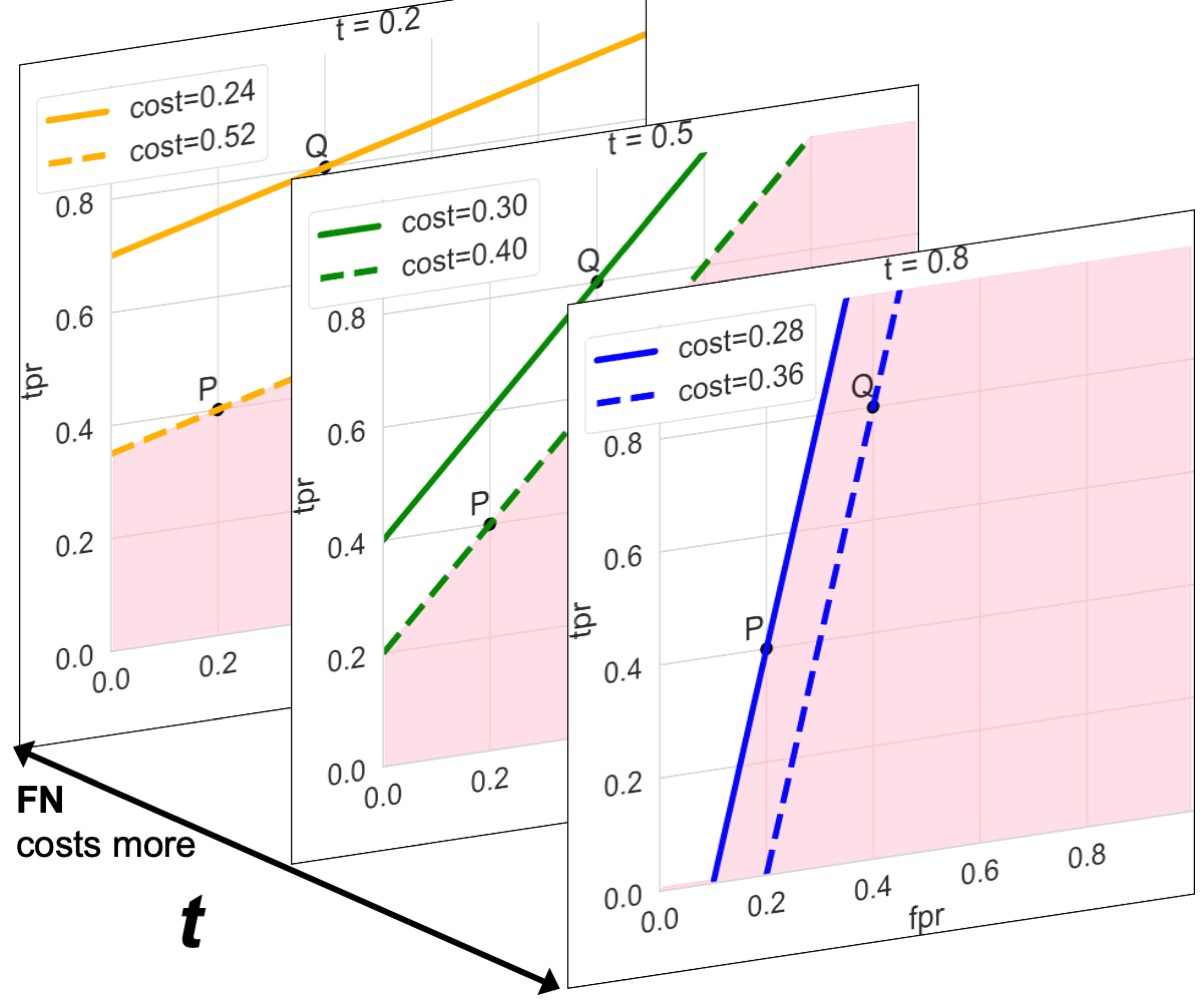}
    & &
    \includegraphics[width=0.4\textwidth]{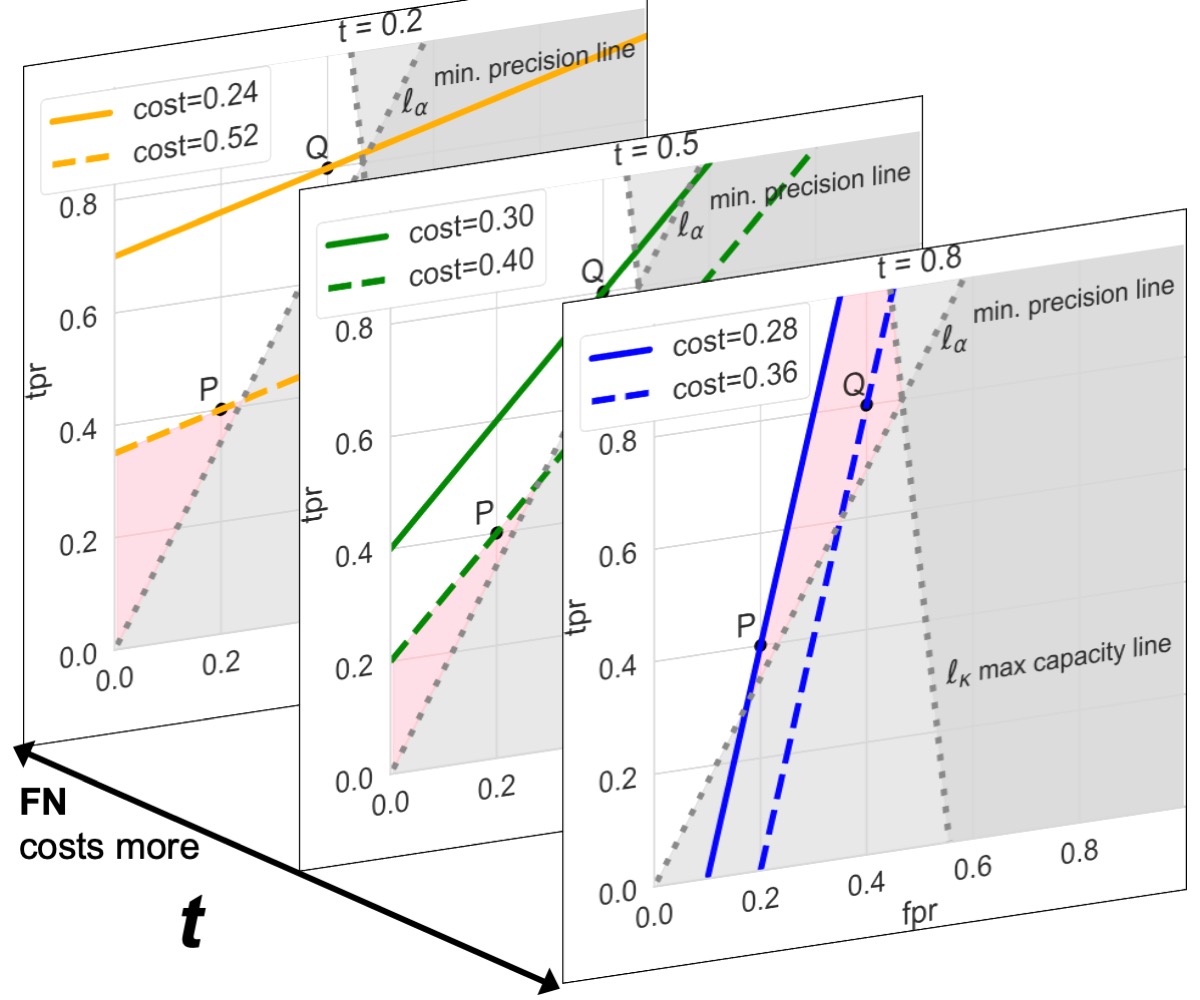}
    \end{tabular}
    \caption{Overview of VOROS (left) and our new partial VOROS (right).
    \emph{Left:} At each fractional cost parameter $t \in [0, 1]$, we  draw iso-performance lines for points $P$ and $Q$ in ROC space. 
    Solid lines have lower cost than dashed lines.
    Points below a line have higher cost than that line.
    A point's \emph{area of lesser classifiers}, colored light pink here for point $P$, is the area below that point's iso-performance line.
    The VOROS~\citep{ratigan2024voros} for a classifier is computed by finding at each $t$ the maximum area of lesser classifiers for any point in its ROC curve, then integrating over a desired distribution $p(t)$. 
    \emph{Right:} Partial VOROS \textbf{excludes} regions in gray that do not achieve a minimum precision $\alpha$ or exceed a maximum capacity $\kappa$ (too many positive predictions).
    These limits correspond to linear constraints in ROC space.
    }
    \label{fig:overview_diagram}
\end{figure*}

\section{BACKGROUND}
\label{sec:background}

\newcommand{\nALL}{N_{*}}
\newcommand{\nTP}{N_{\textnormal{TP}}}
\newcommand{\nFP}{N_{\textnormal{FP}}}
\newcommand{\nFN}{N_{\textnormal{FN}}}
\newcommand{\nTN}{N_{\textnormal{TN}}}
\newcommand{\cFP}{C_{\textnormal{FP}}}
\newcommand{\cFN}{C_{\textnormal{FN}}}
\newcommand{\cRatio}{C_{\textnormal{FP}:\textnormal{FN}}}

We consider an observed dataset $\mathcal{D} = \{X_i, Y_i\}_{i=1}^{\nALL}$ of $\nALL$ total pairs indexed by $i$ of a feature vector $X_i \in \mathcal{X}$ and its associated binary label $Y_i \in \{0,1\}$. Let $\mathcal{P}$ denote the subset of this dataset with positive labels $Y_i=1$, and $\mathcal{N} = \mathcal{D} \setminus \mathcal{P}$ denote the subset with negative labels. 

We wish to use dataset $\mathcal{D}$ to evaluate a score-producing binary classifier $\mathcal{F} : \mathcal{X} \rightarrow \mathbb{R}$. This score can be thresholded to produce a binary prediction $\hat{Y}_i$. We refer to a classifier using threshold $\tau$ to make binary predictions as a \textbf{binarized classifier}, denoted $\mathcal{F}_\tau$. By comparing predictions to true labels over all $\nALL$ examples, we can count the number of true positives $\nTP$, false positives $\nFP$, true negatives $\nTN$, and false negatives $\nFN$.

We now establish concepts needed for a cost-aware ROC analysis, following~\citet{ratigan2024voros}.

\begin{defn}[ROC Space]
The performance of binarized classifier $\mathcal{F}_{\tau}$ on dataset $\mathcal{D}$ at a particular threshold $\tau$ can be represented as a point $(h,k)$ in two-dimensional space where $h$ is the false positive rate and $k$ the true positive rate.
We refer to the set of all possible such $(h,k)$ points, which span the unit square $[0.0,1.0] \times [0.0,1.0]$, as \textbf{ROC space}.
\end{defn}

ROC space is useful for assessing classifier performance with respect to cost.
Let $\cFP$ define the cost of a false positive and $\cFN$ the cost of a false negative.
Given a dataset, the total cost of all possible mistakes is given by $\cFP|\mathcal{N}|+\cFN|\mathcal{P}|$. Naturally, assigning costs to points in ROC space depends on the \emph{ratio} of costs $\cRatio = \frac{\cFP}{\cFN}$ and the relative sizes of $|\mathcal{P}|$ and $|\mathcal{N}|$. Following~\citet{ratigan2024voros}, we can capture this dependency in one parameter $t \in [0.0, 1.0]$.

\begin{defn}[Fractional Cost Parameter]
    Let $t=\dfrac{\cFP |\mathcal{N}|}{\cFP |\mathcal{N}|+\cFN |\mathcal{P}|}$ denote the portion of aggregate misclassification cost due to false positives. We have $0 \leq t \leq 1$.
    \label{defn:t}
\end{defn}
When the dataset $\mathcal{D}$ is fixed, an increase in $t$ implies the cost ratio $\cRatio$ is increasing.
We cannot compare $t$ across datasets with different class balances.

\begin{defn}[Normalized Cost]
    \label{defn:cost}
    The \textbf{normalized cost} of a ROC point $(h,k)$ for dataset $\mathcal{D}$ given parameter $t$ is
    $$\textnormal{Cost}_t(h,k)=t h + (1-t)(1-k)
    = \frac{\cFP \nFP + \cFN \nFN}{ \cFP|\mathcal{N}| + \cFN |\mathcal{P}|}.
    \label{eq:cost}
    $$
\end{defn}
The worst possible point in ROC space, $(1,0)$, would have cost of 1.0. The best point $(0,1)$ would have cost 0.0. All other points have cost in between 0.0 and 1.0, reflecting their cost relative to the worst point.

This formulation of cost as a linear function of ROC coordinates $(h,k)$ is well-known \citep{drummond2000explicitly}. However, \citet{ratigan2024voros} was the first paper to add a separate axis based on the $t$ parameter defined above to ROC space to compare classifiers with respect to variable costs. The key notion here was the ROC surface defined below.

\begin{defn}[ROC Surface]
    Let $(h,k)$ be a point in ROC space, the \textbf{ROC surface} associated to $(h,k)$ is the saddle surface in 3D space with coordinates $x,y,t$ given equivalently by
    $t=\frac{y-k}{y-k+x-h}$ or $ y=\frac{t}{1-t}(x-h)+k$,
    where $t$ is the fractional cost  from Def.~\ref{defn:t}.
\end{defn}

Given a fixed $t$ value, this formulation allows us to find all points $(x,y)$ in ROC space with the same cost as the point $(h,k)$.   
This same-cost set is known as an \emph{iso-performance line} ~\citep{provost2001robust}.

\begin{defn}[Iso-performance Line]
Let $(h,k)$ be a point in ROC space and let $t$ be a fixed fractional cost parameter. Then the line in ROC space
$$y=\frac{t}{1-t}(x-h)+k$$
represents all points $(x,y)$ with the same cost as $(h,k)$ using $t$ and is called an \textbf{iso-performance line}.
\end{defn}
Each ROC panel in Fig.~\ref{fig:overview_diagram} visualizes for a specific $t$ the iso-performance lines for the same two points $P$ and $Q$. Varying $t$ adjusts the iso-performance line's slope.

\section{BOUNDS ON ROC SPACE}

A common criticism of ROC space and the area under the ROC curve is that it fails to measure the performance of a binary classifier under appropriate operational constraints. These criticisms still hold for the ROC surface and the volume over the ROC surface measure of \cite{ratigan2024voros}. In this section we introduce two conditions -- a bound on precision and a bound on capacity -- that restrict the allowed classifiers in ROC space.
These constraints are motivated by early warning classifiers for hospitalized patient monitoring~\citep{harrisonAutomatedSepsisDetection2015,rath2022optimizing}.
We then define a portion of ROC space meeting these constraints and other key assumptions.

\subsection{Precision Bound}

\begin{defn}[Precision]
The precision of the binarized classifier $\mathcal{F}_{\tau}$ on dataset $\mathcal{D}$ is the fraction of positive predictions that are correct:
$$\text{Prec} =\dfrac{\nTP}{\nTP+\nFP}$$
\end{defn}
Precision is also known as positive predictive value.

\begin{defn}[Precision Bound]
To be feasible, a classifier must satisfy a \textbf{minimum precision bound}:
$$
\text{Prec} \geq \alpha
$$
\end{defn}
\vspace{-3mm}
Here, the desired precision $\alpha > 0$ can be set by talking with stakeholders about their tolerance for false alarms.
This constraint is motivated by applications of ML to early warning alert systems, especially in hospitals~\citep{harrisonAutomatedSepsisDetection2015,rath2022optimizing}.

For dataset $\mathcal{D}$, let $p=\frac{|\mathcal{P}|}{|\mathcal{D}|}$ define the \textbf{prevalence}, the fraction of all examples that are positive. Using $p$, then we can recast the definition of precision itself, as well as the bound above, in terms of ROC coordinates.

\begin{lem}[Precision Bound in ROC Space]
Let $\mathcal{F}_{\tau}$ be a binarized classifier with ROC coordinates $(x,y)$ on a dataset with prevalence $p$. Then
%Then its precision is at least $\alpha$ iff 
$$ \text{Prec} \geq \alpha ~~~\text{if}~~~ y \geq \dfrac{\alpha(1-p)}{(1-\alpha)p}x.$$
We call the line in ROC space where this holds with equality the \textbf{minimum precision line} $\ell_{\alpha}$.
\end{lem}

\begin{proof}
%Starting from the original precision bound, 
We write precision in terms of $x,y$ and $p$, using $\nTP = |\mathcal{P}|y$ and $\nFP = |\mathcal{N}|x$. Then, solve for $y$:
$$\dfrac{py}{py+(1-p)x} \geq \alpha
~\Rightarrow~
y \geq \dfrac{\alpha(1-p)}{(1-\alpha)p}x.
$$
This algebra is valid when $1{-}\alpha > 0, p > 0$.
\end{proof}

This inequality means that if we require a priori that classifiers have at least $\alpha$ precision, then rather than the entire unit square of ROC space, we only consider the region above the minimum precision line $\ell_{\alpha}: y=\frac{\alpha(1-p)}{(1-\alpha)p}x$ through the origin. When $p$ is small or when $\alpha$ is large, then we ignore a large portion of ROC space.

\subsection{Capacity Bound}

Beyond precision, another issue prevalent in applications is that the system that responds to alerts (positive predictions) has a maximal \emph{capacity}. Unlike precision, which is a rate, capacity is often an absolute cutoff due to limited resources. In hospital alert systems, capacity constraints arise due to limited time to tend to patients by existing staff.  In information retrieval, there may be limited time to handle relevant documents.
Throughout, we assume that \emph{not alarming} has no impact on capacity. This is reasonable in hospitals, as no resources beyond standard care need be allocated when there is no alarm.

\begin{defn}[Capacity Bound]
To be feasible, the total number of predicted positives a classifier produces must not exceed a provided \textbf{maximum capacity}:
\begin{align*}
\nTP + \nFP \leq \kappa     
\end{align*}
\end{defn}
Here, the value of $\kappa > 0$ can be set by stakeholders, indicating the maximum number of alerts that can be handled by the system given typical resources.

\begin{lem}[Capacity Bound in ROC Space]
Let $\mathcal{F}_{\tau}$ be a binarized classifier with ROC coordinates $(x,y)$ for dataset $\mathcal{D}$. If capacity bound $\kappa$ is satisfied, then

\begin{enumerate}[noitemsep,leftmargin=*]
\item The number of predicted positives is: $|\mathcal{P}|y+|\mathcal{N}|x$. 
\item We have $y \leq \textstyle \frac{1}{|\mathcal{P}|}(\kappa - |\mathcal{N}|x)$.
\end{enumerate}
\end{lem}
We call the line in ROC space where this holds with equality the \textbf{maximum capacity line}, $\ell_{\kappa}$.

\begin{proof}
The number of predicted positives $\nTP {+} \nFP$ is $|\mathcal{P}|y{+}|\mathcal{N}|x$ by definition. Bound satisfaction means $|\mathcal{P}|y+|\mathcal{N}|x \leq \kappa$. Solving for $y$ completes the proof.
\end{proof}

\subsection{Feasible Region}

We now consider enforcing both precision and capacity constraints, along with some practical assumptions typical in our target applications. These constraints narrow down ROC space to a particular smaller region we call the feasible region. This is similar to what \citet{morasca2020assessment} call a ``region of interest'', though ours incorporates precision and capacity.

\parbox{0.98\linewidth}{
\begin{defn}[Practical Assumptions] In order to focus on the cases that arise in practical settings, we make some assumptions going forward (remaining edge cases are handled in App.~\ref{supp_casesForPolygons}). From here on, in the main paper, we assume our dataset, precision limit $\alpha$, and capacity limit $\kappa$ satisfy:
\begin{itemize}[leftmargin=4mm]
    \item $|\mathcal{P}| < |\mathcal{N}|$: Negatives are more common.
    \item $p<\alpha < 1.0$ : Precision is reasonable. 
    \item $0 < \kappa < |\mathcal{D}|$: Capacity is non-trivial.
    \item $t < \frac{\alpha |\mathcal{N}|}{\alpha |\mathcal{N}| + (1-\alpha)|\mathcal{P}|}$: Never-alarm has maximal cost.
\end{itemize}
\label{def:practical_assumptions}
\end{defn}
}
The strict inequality $\kappa <|\mathcal{D}|$ in item 3 ensures the always-alarm baseline is not feasible. The $t$ bound in the last item ensures that the never-alarm baseline maximizes cost for the feasible region (see App.~\ref{supp_isoperformanceCases}).

\parbox{0.97\linewidth}{
\begin{defn}[Feasible Classifier]
Let $\mathcal{F}_{\tau}$ be a classifier with ROC coordinates $(h,k)$. Given a dataset, precision limit $\alpha$, and capacity limit $\kappa$ that satisfy Def.~\ref{def:practical_assumptions}, we call $\mathcal{F}_{\tau}$ a \textbf{feasible classifier} and $(h,k)$ a \textbf{feasible point} if
$
\text{Prec} \geq \alpha ~\text{and}~ 
\nTP {+} \nFP \leq \kappa
$.
\end{defn}
}

\parbox{0.97\linewidth}{
\begin{defn}[Feasible Region]
%Let set $S$ define the constraints on classifiers for 
For a fixed dataset and limits $\alpha,\kappa$, the \textbf{feasible region} in ROC space is the set of all feasible points.
\end{defn}
} 

 \begin{figure*}[!t]
 \centering
 \begin{tabular}{c c c} 
     \circled{1} with $\alpha=0.15,  \kappa=900$
     &
     \circled{2} with $\alpha=0.15,  \kappa=3000$
     &
     \circled{3} with $\alpha=0.15,  \kappa=9100 $
     \\
     \includegraphics[width=.22\textwidth,keepaspectratio,clip]{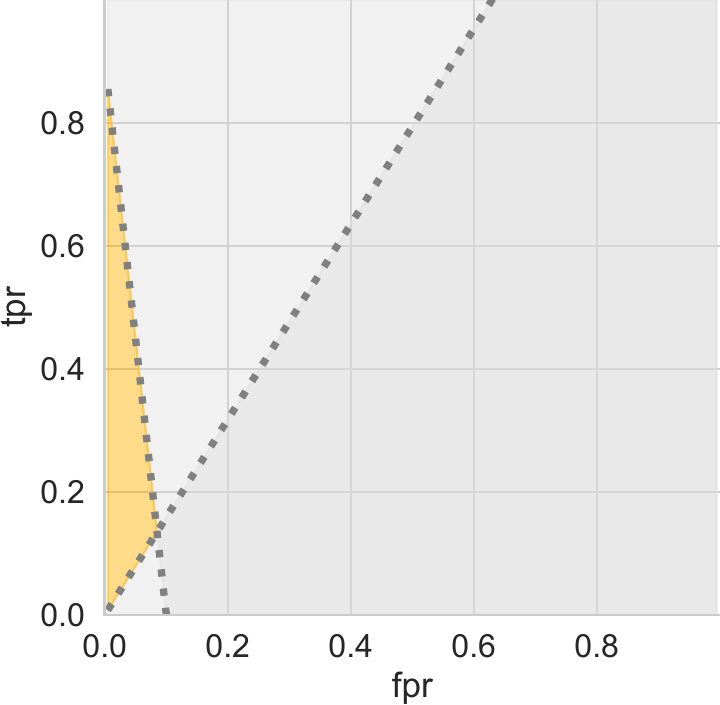}
     &   
     \includegraphics[width=.22\textwidth,keepaspectratio,clip]{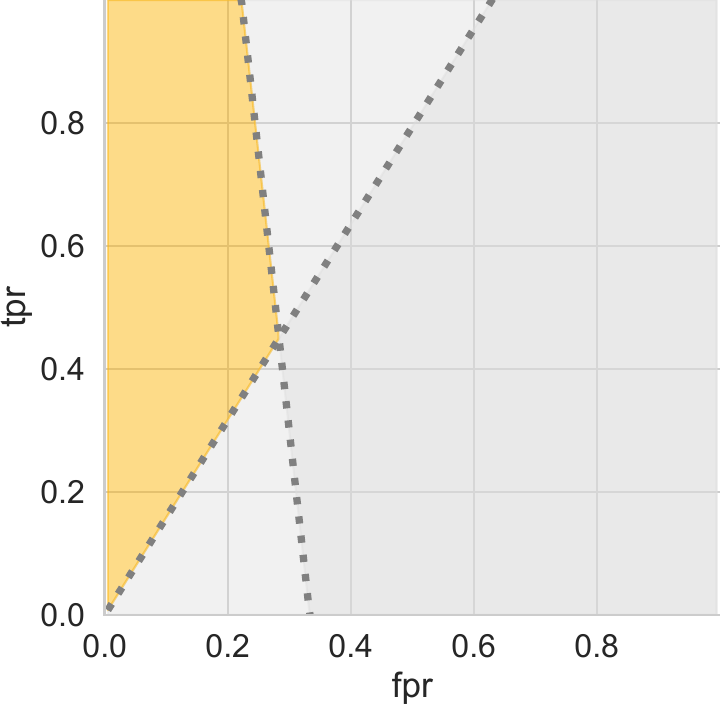}
     &
     \includegraphics[width=.22\textwidth,keepaspectratio,clip]{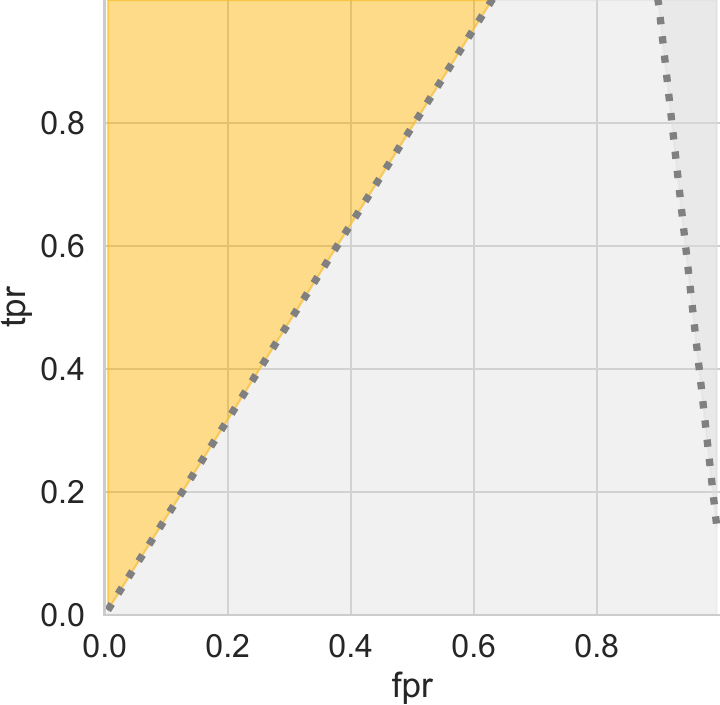} 
 \end{tabular}
 \caption{
     Cases for distinct feasible region polygons (colored in gold) under the practical conditions of Def.~\ref{def:practical_assumptions}.
     \\We set $|\mathcal{P}|=1000$, $|\mathcal{N}|=9000$ so $p = 0.1$, roughly matching the mortality alert applications in Sec.~\ref{sec:mimiciv_experiments}.
 }%endcaption
 \label{fig:polygon_cases}
 \end{figure*}
 
\subsection{Geometry of the Feasible Region}
\label{sec:geom_feasible_region}

We can define the geometry of the feasible region via the minimum precision line $\ell_{\alpha}$ and the maximum capacity line $\ell_{\kappa}$.
The line $\ell_{\alpha}$ has its y-intercept at the origin, and slope within $(1.0, +\infty)$ for $\alpha \in (p, 1.0)$ by Def.~\ref{def:practical_assumptions}.
Similarly, $\ell_{\kappa}$ has fixed slope $-\frac{|N|}{|P|}$ (always ${<}-1$), with free y-intercept $\frac{\kappa}{|\mathcal{P}|}$ for $\kappa \in (0, |\mathcal{D}|)$. Where these lines intersect determines what kind of polygon the feasible region forms. 

The 3 cases below exhaustively cover all possible $\kappa$ for a given feasible value of $\alpha \in (p, 1.0)$ subject to the practical assumptions made in Definition~\ref{def:practical_assumptions}, as diagrammed in Figure~\ref{fig:polygon_cases}. 
\begin{itemize}
    \item {\small \circled{1}} $0 < \kappa < |\mathcal{P}|$: Triangle that excludes (0,1). Here, $\ell_{\alpha}$ and $\ell_{\kappa}$ intersect inside ROC space. The capacity bound eliminates the perfect classifier.
    \item {\small \circled{2}} $|\mathcal{P}| \leq \kappa < \frac{1}{\alpha}|\mathcal{P}|$. Quadrilateral including (0,1). $\ell_{\alpha}$ and $\ell_{\kappa}$ intersect inside ROC space, with capacity allowing the perfect classifier.
    \item {\small \circled{3}} $\frac{1}{\alpha}|\mathcal{P}| \leq \kappa < |\mathcal{D}|$. Right Triangle including (0, 1). $\ell_{\alpha}$ and $\ell_{\kappa}$ intersect above ROC space (above the $y{=}1$ upper edge), so the capacity constraint is made redundant and precision alone dominates. 
\end{itemize}

Edge cases that violate the assumptions of Def.~\ref{def:practical_assumptions} are handled in App.~\ref{supp_casesForPolygons}. Focusing on 3 main cases here simplifies presentation for most practical situations.

We can now define the polygon enclosing the feasible region in terms of specific vertices for each case:
\parbox{0.98\linewidth}{
\begin{defn}[Notation for Vertices]
Let $i,j \in \{0,1\}$ and let $\beta \in \{\alpha, \kappa\}$. Define the following 9 vertices for all possible $i,j,\beta$
\begin{itemize}[leftmargin=12mm]
\item[(1-4)] $v_{ij}$ is the intersection of $x=i$ and $y=j$
\item[(5-6)] $v_{\beta j}$ is the intersection of $\ell_\beta$ and $y=j$
\item[(7-8)] $v_{i\beta}$ is the intersection of $x=i$ and $\ell_\beta$.
\item[(9)] $v_{\alpha \kappa}$ is the intersection of $\ell_\alpha$ and $\ell_\kappa$.
\end{itemize}
\label{def:vertices}
\end{defn}
}

\parbox{0.98\linewidth}{
\begin{defn}[Feasible Region Polygon]
Moving clockwise from the origin, the bounding vertices of the region of interest are $v_{00}v_{0 \kappa}v_{\alpha \kappa}$ in case 1, $v_{00}v_{01}v_{\kappa 1}v_{\alpha \kappa}$ in case 2, and $v_{00}v_{01}v_{\alpha 1}$ in case 3.
\end{defn}
}

\textbf{Area of Feasible Region.}
We can compute the area of the feasible region by applying the well-known shoelace formula~\citep{leeShoelaceFormulaConnecting2017,zwillingerCRCPolygonFormula2018}, which computes the area of a polygon given its boundary vertices (see App.~\ref{supp:geometry_cases_satisfying}).

\textbf{Dependence on Class Balance.} The slopes of both lines $\ell_{\alpha}$ and $\ell_{\kappa}$ and thus the feasible region's area depends explicitly on the sizes of $\mathcal{P}$ and $\mathcal{N}$. These sizes must be fixed and known prior to model evaluation. 
This dependence is fundamental to evaluating precision and capacity; neither can be checked solely from a normalized confusion matrix. 
% Drifts in prevalence can affect the calculation of the partial VOROS at all stages of the calculation.

\section{PARTIAL AREA AND VOLUME}
\label{sec:partial_area_volume}

\begin{figure*}[!t]
%\includesvg[width=2\columnwidth]{figures/equation(2).svg}
    \includegraphics[width=2\columnwidth]{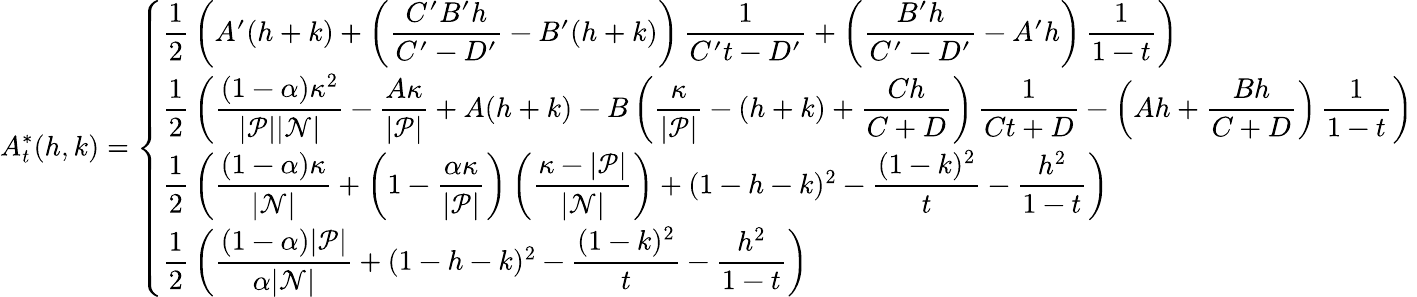}
\caption{Formula for partial area of lesser classifiers for a given $(h,k)$ location in ROC space and cost parameter $t \in [0, 1]$. The 4 cases here depend on the geometry of the feasible region and the iso-performance line $\ell_t$ through $(h,k)$. The constants $A$,$B$,$C$, $D$, $A'$,$B'$, $C'$, and $D'$ are explained in Lemma~\ref{lem:partial_area_form} and derived explicitly in Lemma~\ref{lem:formula_for_partial_area}.
 The cases above are, in order, the areas of $v_{00}v_{\alpha t}v_{0t}$, $v_{00}v_{\alpha\kappa}v_{\kappa t}v_{0t}$, $v_{00}v_{\alpha \kappa}v_{\kappa 1}v_{t1}v_{0t}$, and $v_{00}v_{\alpha 1}v_{t1}v_{0t}$ respectively, using the vertex definitions in Def.~\ref{def:vertices}.
App.~\ref{supp_isoperformanceCases} identifies which case is needed for given inputs.
}
\label{fig:equation_partial_area}
\end{figure*}

%\parbox{0.97\linewidth}{
We now study how to perform \emph{cost-aware} ranking of classifiers in the feasible region.
\citet{ratigan2024voros} introduced two key ideas for cost-aware ranking in their unconstrained setting: 
the notion of a lesser classifier and the area of lesser classifiers in ROC space.
In Sec.~\ref{subsec:area}, we extend these ideas to enforce a minimum precision $\alpha$ and a maximum capacity $\kappa$. Next, in Sec.~\ref{subsec:volume}, we explain how to lift this analysis to 3D space $(x,y,t)$. Finally, we show how to integrate over a given distribution $p(t)$ to compute a partial volume.

\subsection{Partial Area of Lesser Classifiers}
\label{subsec:area}

\parbox{0.97\linewidth}{
\begin{defn}[Lesser Classifier]
Let $\mathcal{F}_1, \mathcal{F}_2$ be feasible binarized classifiers with ROC coordinates $(h_1,k_1),(h_2,k_2)$ on dataset $\mathcal{D}$. Then $\mathcal{F}_1$ is a \textbf{lesser classifier} of $\mathcal{F}_2$ at cost parameter $t$ if it has worse cost: $\text{Cost}_t(h_1,k_1) > \text{Cost}_t(h_2, k_2)$.
\label{lesser}
\end{defn}
}

\parbox{0.97\linewidth}{
\begin{defn}[Partial Area of Lesser Classifiers]
Let $\mathcal{F}_\tau$ be a feasible classifier and let $t$ be a fractional cost parameter. The \textbf{partial area of lesser classifiers}, $A_t^*(\mathcal{F}_{\tau})$ is the area of the portion of the feasible region of ROC space consisting of $\mathcal{F}_{\tau}$'s  lesser feasible classifiers using cost parameter $t$.
\end{defn}
}

\parbox{0.97\linewidth}{
By Definition~\ref{lesser} the partial area of lesser classifiers is cost monotone. Under some values of $\alpha$ and $\kappa$, feasible classifiers may occupy only a small region of ROC space. The numerical value of the partial area may thus be small even for the best feasible classifiers. For a metric whose value is easy to interpret as good or bad regardless of $\alpha,\kappa$, we recommend a normalization:
}

\parbox{0.97\linewidth}{
\begin{defn}[Normalized Partial Area]
Let $\mathcal{F}_\tau$ be a feasible classifier and let $t$ be a fractional cost parameter. The \textbf{normalized partial area} of lesser classifiers is the ratio of the partial area $A^*_t( \mathcal{F}_\tau )$ to the partial area of all feasible points in ROC space.
\end{defn}
}

We can compute the denominator in this ratio directly using the area of the feasible region (see App.~\ref{supp:geometry_cases_satisfying}).
If the perfect classifier located at $(0,1)$ in ROC space is feasible, then this denominator is \emph{equivalent} to the partial area of lesser classifiers for the perfect classifier.

Our notion of normalized partial area of lesser classifiers is similar in spirit to the ratio of relevant areas defined by~\citet{morasca2020assessment}. 

\begin{lem}[Form of the Partial Area]
The partial area of lesser classifiers is a sum of quotients of linear functions of $t$ as given in Figure 3.
\label{lem:partial_area_form}
\end{lem}
\begin{proof}
A full proof is in App.~\ref{supp_isoperformanceCases}. The area depends on how the iso-performance line $\ell_t$ through $(h,k)$ intersects the feasible region. The key insight is that the $x$-coordinate of $\ell_t \cap \ell_{\alpha}$ can be expressed as $A'+\frac{B'}{C't-D'}$ and that of $\ell_t \cap \ell_{\kappa}$ as $A+\frac{B}{Ct+D}$, for suitable scalars $A,B,C,D, A', B', C', D'$.
\end{proof}

\subsection{Partial Volume over the ROC surface}
\label{subsec:volume}

When stakeholders have a range or distribution of cost parameters $t$ in mind as well as many possible thresholds $\tau$, \citet{ratigan2024voros} showed how integrating the area of lesser classifiers over that range of $t$ leads to a  performance metric called the volume over the ROC surface (VOROS). 
Importantly, at each $t$, they sensibly use the largest area over all thresholds $\tau$.
We now extend this idea to a \emph{partial volume} that only considers the feasible region imposed by precision and capacity constraints.

\begin{defn}[Partial VOROS or PV]
Let $\mathcal{F}$ be a score-producing classifier, let $p(t)$ be a valid probability density function over [0,1] for the cost parameter $t$, and let $\alpha,\kappa$ define precision and capacity limits so that the feasible region is well-defined with area $A^*$. 
Then the \textbf{partial volume over the ROC surface} (partial VOROS) is the normalized partial area of lesser classifiers, averaged over the provided cost distribution:
$$PV(\mathcal{F}) = \dfrac{1}{A^*} \int_{t=0}^1 p(t) \max_{\tau} (A_t^*(\mathcal{F}_{\tau})) dt.
$$
\parbox{0.97\linewidth}{
Here, the maximum is taken over the set of thresholds for $\mathcal{F}$ that produce distinct feasible points $(h,k)$.
}
\end{defn}

%% MCH: Kept old version just in case
% \begin{defn}[Partial VOROS]
% Let $\mathcal{F}$ be a score-producing classifier, let $[a,b] \subseteq [0,1]$ be a cost parameter range, and let $\alpha,\kappa$ define precision and capacity limits so that the feasible region is well-defined with area $A^*$. Then the \textbf{partial volume over the ROC surface (VOROS)} is the normalized partial area of lesser classifiers, averaged over the provided range:
% $$PV(\mathcal{F}) = \dfrac{1}{(b-a)A^*}\int_{t=a}^b\max_{\tau} (A_t^*(\mathcal{F}_{\tau})) dt.
% $$
% Here, the maximum is taken over the set of thresholds for $\mathcal{F}$ that produce distinct feasible points $(h,k)$.
% \end{defn}

Overall, PV is a higher-is-better performance metric.
Its range is between 0.0 and 1.0 regardless of the dataset or constraints $\alpha,\kappa$. 
The best possible PV is 1.0, achieved by the highest feasible classifier on the $y$-axis.
The worst value of 0.0 comes from the never-alarm baseline at (0,0), assuming Def.~\ref{def:practical_assumptions} holds.

PV can be viewed as an expectation over the distribution $p(t)$,
$PV = \frac{1}{A_*} \mathbb{E}_{p(t)}[ \max_{\tau} (A_t^*(\mathcal{F}_{\tau}))  ]$.
Users could simply set $p(t)$ to a uniform distribution over the fractional cost $t$, perhaps confined to a task-relevant interval $[a,b] \subseteq [0, 1]$. Stakeholders may prefer to define a distribution over cost ratio, $p(\cRatio)$, which implies a distribution over $t$ via Def.~\ref{defn:t}.
The User Guide (App.~\ref{supp:user_guide}) gives practical advice on computing PV.

\textbf{Runtime.}
The maximum in the definition is computable in $O(h_c)$-time from the convex hull of a ROC curve, where $h_c$ is the number of feasible points in the curve's convex hull and thus the number of possible thresholds. 
See App.~\ref{supp_compComplexity} for an algorithm.
Recall that $A^*_t$ is monotonically decreasing in $\text{Cost}_t$. We use this fact to speed up our calculation.
Instead of directly assessing area $A^*_t(\mathcal{F}_{\tau})$ at each of the $h_c$ thresholds $\tau$,
we can find the threshold that minimizes $\text{Cost}_t$. Thus, only one area calculation is needed per $t$ value.

\textbf{PV and ROC Dominance.}
Similar to the VOROS and the traditional area under the ROC curve, the partial VOROS will always assign a higher value to a ROC curve that dominates.

\parbox{0.97\linewidth}{
\begin{lem}[Dominance]
Let $y=f_1(x)$, $y=f_2(x)$ be two ROC curves. If for all $x$, $f_1(x) \geq f_2(x)$, then the VOROS and partial VOROS of $f_1$ will be higher than that of $f_2$.
\end{lem}
}
\begin{proof}
Since the curve for $f_1(x)$ is above and to the left of the curve $f_2(x)$, for any fractional cost parameter $t$, the best performing point on $f_2(x)$ will be in the area of lesser classifiers of some point on $f_1(x)$. 
\end{proof}

\begin{figure}[!t]
\begin{tabular}{c}
\includegraphics[width=.45\textwidth]{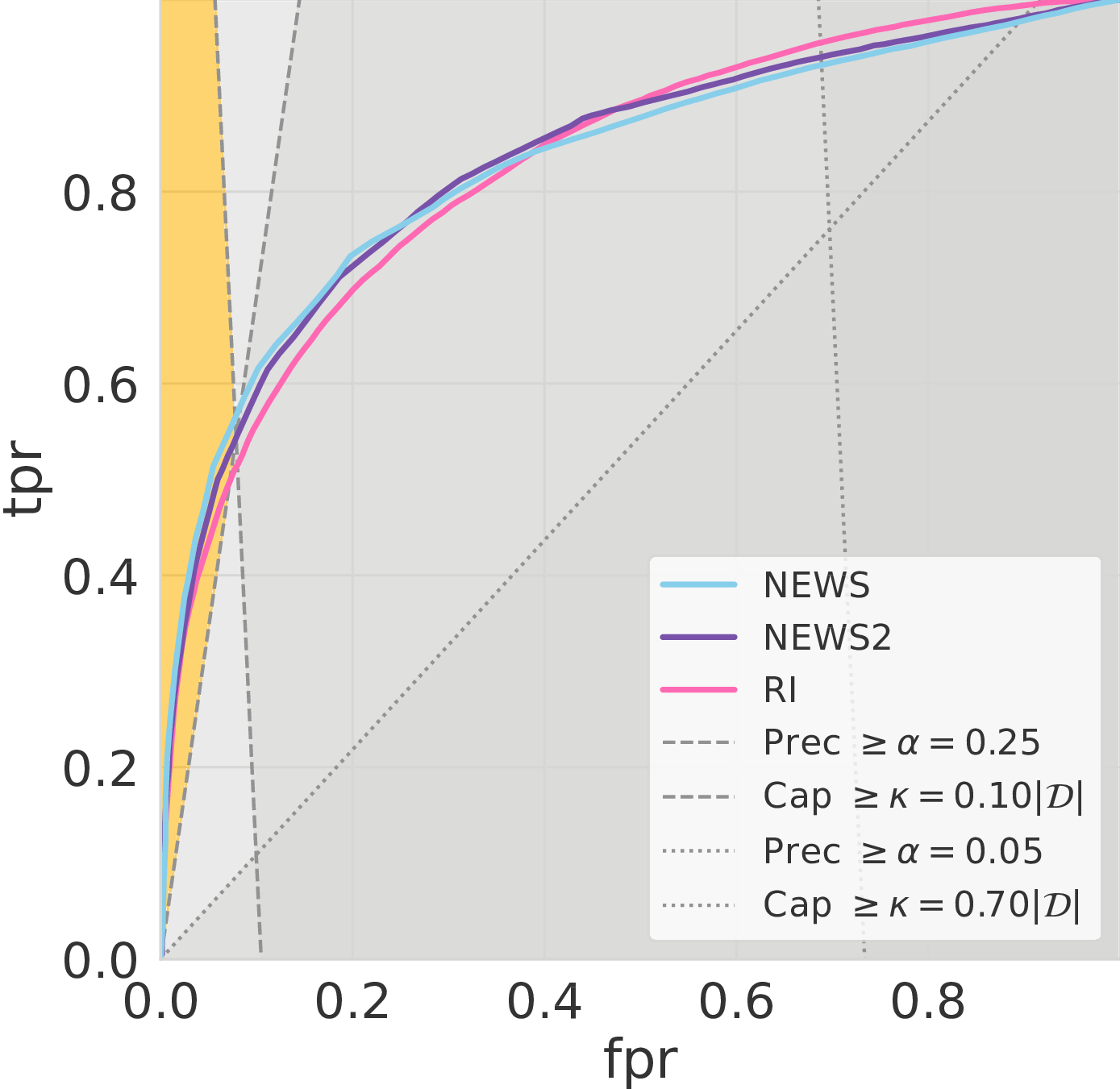}
\\
\includegraphics[width=.45\textwidth]{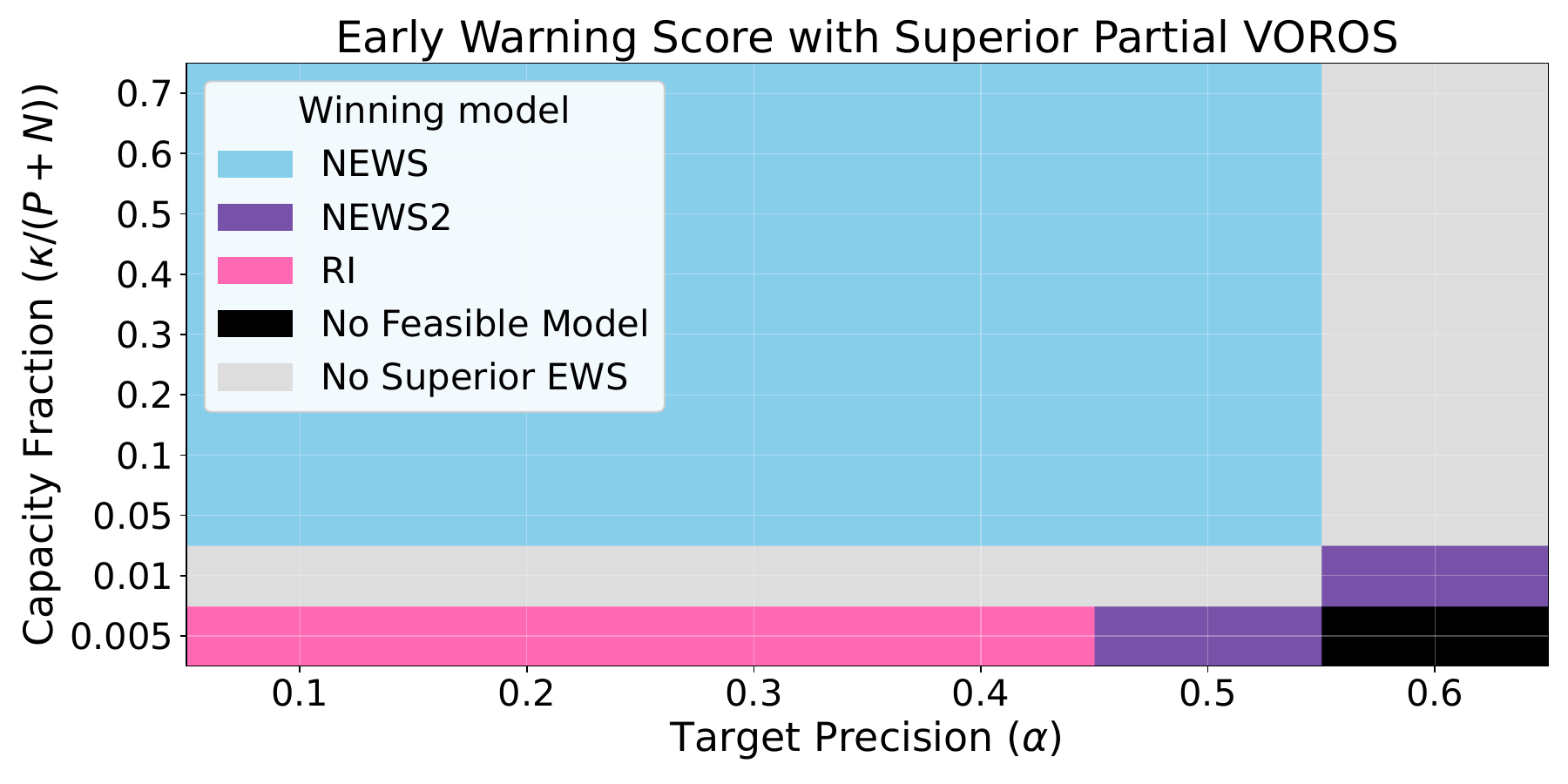}
\end{tabular}
\vspace{-3mm} % HACK WHITESPACE
\caption{
\textbf{Re-analysis of Early Warning Scores for Deterioration}, using ROC curves from \citeauthor{edelsonEarlyWarningScores2024}. 
Top: ROC curves with strict (-- lines) and more lenient (:) constraints. 
Bottom: Heatmap of showing which model has the best PV score over possible $\alpha, \kappa$ limits.
}%endcaption
\label{fig:results_edelson}
\end{figure}

\section{MODEL EVALUATION WITH PV}
\label{sec:edelson_experiments}
We examine partial VOROS as a post-hoc metric for existing clinical early warning scores. We re-analyze ROC curves published by \citet{edelsonEarlyWarningScores2024}, measuring the performance of widely-used risk scores on 362,926 patient stays in 7 hospitals in Connecticut, USA. The outcome of interest is deterioration, meaning transfer to the ICU or death within 24 hours. The prevalence $p$ is 4.6\% in \citeauthor{edelsonEarlyWarningScores2024}'s data.

For this case study, we focus on 3 scoring methods: NEWS~\citep{royalcollegeofphysiciansNationalEarlyWarning2012},
NEWS2~\citep{royalcollegeofphysiciansNationalEarlyWarning2017},
and the Rothman Index~(RI, \citet{rothmanDevelopmentValidationContinuous2013}).
We visualize ROC curves in Fig.~\ref{fig:results_edelson}: each curve crosses others, and it is thus imperative to consider costs and constraints to determine the best operating point of all methods.
We consider capacity $\kappa$ from $0.5\% - 70\%$ of $|\mathcal{D}|$ and  precision limits $\alpha \in (0.1, 0.6)$.
When calculating partial VOROS, we assume a uniform distribution over cost ratios $C_{\text{FP:FN}}$ that map to $t \in (0.01, 0.7)$, so that Def.~\ref{def:practical_assumptions} is satisfied. 

Over a range of $\alpha, \kappa$ limits, we plot a heatmap in Fig.~\ref{fig:results_edelson} indicating where different scoring methods have clear wins and regions where they are indistinguishable (normalized PV is within 0.01).
We find that NEWS2 preforms best at high $\alpha$ and low capacity, RI has superior recall at lower $\alpha$ and low capacity, and NEWS performs best once capacity is greater than 5\% of all examples.
\citet{edelsonEarlyWarningScores2024} previously concluded that NEWS ``outperformed'' the two other scores here. Our analysis with partial VOROS adds nuance: Each method could be a clear winner depending on precision and capacity limits.

\section{MODEL SELECTION WITH PV}
\label{sec:mimiciv_experiments}
We now illustrate how our PV metric can be used for model selection, not just post-hoc evaluation.
We will show how different performance metrics can lead to distinct selections for the best ROC curve and best threshold on that curve on a  healthcare modeling task. 
We focus on in-hospital mortality prediction with the MIMIC-IV~\citep{johnsonMIMICIVFreelyAccessible2023}
and eICU~\citep{pollardEICUCollaborativeResearch2018} datasets.
We aim for competitive \emph{reproducible} classifiers, not state-of-the-art on past benchmarks.
We'll assess how cost-aware selection strategies like our PV perform against cost-unaware alternatives, in terms of the ultimate cost on test data.

\textbf{Task Description.}
For the MIMIC-IV task, we follow~\citeauthor{rath2022optimizing}. 
At each ICU patient-stay we observe basic facts at admission (age, insurance type, weight, gender) as well as time-varying signals over the first 48 hours (6 vitals and 7 labs).
%App.~\ref{supp_mimiciv}),
The prediction task is to identify if the patient will die during the rest of their hospital stay.
We use an open-source pipeline~\citep{guptaExtensiveDataProcessing2022} to obtain a cohort of patient-stays, each represented as a 0-48 hr time series with measured values every 2 hours.
We keep only stays produced by \citeauthor{guptaExtensiveDataProcessing2022}'s code with at least 3 vitals or labs measured in the last 16 hours.
The eICU data task is also mortality prediction after 48 hours. Detailed setups for both tasks are in App.~\ref{supp_mimiciv}.

%Out[33]: ((15474, 14), (7802, 14), (7861, 14))
% In [30]: te_df['mortality_outcome'].mean()
% Out[30]: 0.10749268540898105
% In [31]: va_df['mortality_outcome'].mean()
% Out[31]: 0.10305049987182774
% In [32]: df['mortality_outcome'].mean()
% Out[32]: 0.10372237301279566
For MIMIC-IV, our train/valid./test sets contain 15474/7802/7861 patient-stays with prevalence $p$ of .104/.103/.107. This split was done by patient-id in a label-stratified way. We favor  larger test sets than usual to be sure we can measure precision well despite low prevalence. 
For eICU, we have 38532/19267/19267 ICU-stays with prevalence $p$ of 0.116/0.114/0.114.
%Our train/valid./test sets contain 15474/7802/7861 patient-stays with prevalence $p$ of .104/.103/.107. This split was done by patient-id in a label-stratified way. We favor  larger test sets than usual to be sure we can measure precision well despite low prevalence. 

\textbf{Featurization.}
For MIMIC-IV, for each of 13 time-varying univariate channels, we extract features that represent 7 summary functions (was ever measured, time since last measurement, mean, variance, min, median, max, slope) over 3 windows (0-48 hours, 24-48 hours, 32-48 hours). Missing values in extracted features are filled with the population mean, then rescaled to the 0.0-1.0 range.
For eICU, we use similar processing of the features provided by \citet{PhysioNet-mimic-eicu-fiddle-feature-1.0.0}.

% LR: 80 = 4 * 4 * 5
% MLP: 24 = 2 * 2 * 3 * 2
% RF: 36 = 4 * 3 *3 
\textbf{Classifiers.}
We examine logistic regression, multi-layer perceptron, and random forests, using sklearn~\citep{pedregosaScikitlearnMachineLearning2011}. 
For each, we conduct an extensive grid search of 24-80 hyperparameter configurations to avoid overfitting (details in App.~\ref{supp_mimiciv}).
% Broadly, we find our reported AUROCs around 0.8 match or beat similar models trained by \citeauthor{guptaExtensiveDataProcessing2022} or \citeauthor{rath2022optimizing}.

\textbf{Scenarios for Costs and Constraints.}
We consider two scenarios for costs and constraints.
In each one, we fix particular $\alpha, \kappa$ bounds, and define a distribution over the cost ratio $\cRatio$, which implies a distribution over $t$ via the invertible mapping in Def.~\ref{defn:t}.
\begin{enumerate}[leftmargin=*]
    \item The first scenario is relatively permissive. We fix $\alpha=0.15, \kappa = 0.5|\mathcal{D}|$, and model the cost ratio as $\cRatio \sim \text{Unif}(\frac{1}{9}, \frac{1}{6})$. Here, false negatives are more costly than false positives, but by less than 10x.
    On MIMIC-IV, this implies a non-uniform density for $t \in (0.5, 0.6)$ for computing PV.
    \item The second scenario is more constrained. We set $\alpha=0.5,\kappa=0.1|\mathcal{D}|$ to reflect more stringent clinical priorities for high precision and low total alarms. Recent economic estimates~\citep{rogers_optimizing_2023} suggest a missed clinical deterioration may have roughly 20-40 times the cost of a false alarm, so we use $\cRatio \sim \text{Unif}(\frac{1}{40}, \frac{1}{20})$. On MIMIC-IV data, this implies a non-uniform density for $t \in (.18, .31)$.
\end{enumerate}

\textbf{Evaluation Plan.}
We compare 4 possible selection strategies: max pAUROC and max recall in the feasible region ($\alpha,\kappa$-aware; $t$-unaware), max VOROS ($\alpha,\kappa$-unaware; $t$-aware), and max PV (aware of both constraints and costs). For each strategy, we search a grid of over 140 ROC curves each from a different classifier-hyperparameter combination.
For each possible strategy, we select one curve, and then a threshold $\tau$ to use at each $t$, always using the validation set for selection decisions (details in App.~\ref{supp_mimiciv}).
We then evaluate and report the test set performance of each selected threshold. In this way, we examine how binary alert systems behave on new data, avoiding unrealistic post-hoc threshold selection on the test set. 

\textbf{Results.}
For MIMIC-IV, Tab.~\ref{tab:mimiciv_results_main} reports the test cost of each strategy across the two scenarios.
These are the costs \emph{in expectation} over a uniform distribution of cost ratios $C_{\text{FN:FP}}$. 
eICU cost results are in Tab.~\ref{tab:eICU_results_main}.
Further results are in App.~\ref{supp_mimiciv}, including visuals of selected ROC curves, expanded tables, and details of the calculation of expected cost.
% Fig.~\ref{fig:results_mimiciv} gives results on MIMIC-IV. Panel (a) shows two ROC curves selected by different strategies: Curves \emph{crossing} is a sign that careful selection is needed. 
% Panel(b) shows their respective normalized partial areas as a function of $t$. 

\textbf{Analysis.}
In the more relaxed scenario 1, our max partial VOROS (PV) strategy yields better (lower) normalized cost on test data than 3 alternatives on MIMIC-IV. It also ties for best cost on eICU. This highlights how using cost-unaware strategies like maximizing recall can be inadequate, even if they incorporate ``hard cutoffs'' for the constraints of interest.

The more stringent scenario 2 favors the lower left of ROC curves where there is less crossing; we find 3 of the 4 strategies (including PV) sensibly choose the same ROC curve and thus have similar costs. Only \citeauthor{ratigan2024voros}'s original VOROS ($\alpha,\kappa$-unaware) yields a model with worse cost performance in this scenario on both datasets, highlighting the inadequacy of \emph{ignoring constraints} in model selection.

On MIMIC-IV, all selected models satisfy the $\alpha,\kappa$ constraints on test data. However, on eICU several models that satisfy $\alpha,\kappa$ on validation fall slightly short of satisfying the same constraints on test (e.g. precision of 0.44 instead of 0.5), due to statistical variation across subsets of the same size. These are marked with red highlighting in Tab.~\ref{tab:eICU_results_main}; detailed results in Tab.~\ref{tab:eicu_scenario_tables_detailed}.

\definecolor{MyLightRed}{rgb}{0.99,0.45,0.45}

\begin{table}[!t]
\resizebox{0.48\textwidth}{!}{%
\begin{tabular}{l r r}
& scenario 1 & scenario 2
\\
& $\alpha = 0.15, \kappa = 0.5 |\mathcal{D}|$
& $\alpha = 0.5, \kappa = 0.1 |\mathcal{D}|$
\\
& $\frac{\cFP}{\cFN} \sim \text{Unif}(\frac{1}{9}, \frac{1}{6} )$
& $\frac{\cFP}{\cFN} \sim \text{Unif}(\frac{1}{40}, \frac{1}{20} )$
\\
& $t \in [.5,.6]$
& $t \in [.18,.31]$
\\
\hline
Strategy 
 & Avg. Cost & Avg. Cost
\\
\hline
max VOROS
& \textbf{0.261}
& 0.640
\\
max pAUROC
& 0.303
& \textbf{0.538}
\\
max recall 
& 0.302
& \textbf{0.538}
\\
max PV (ours)
& \textbf{0.261}
& \textbf{0.538}
\end{tabular}
}%end resizebox
\caption{
\textbf{MIMIC-IV results for model selection.}
We report average costs on test (lower is better).
All models satisfy the $\alpha,\kappa$ constraints on both validation and test. Further results in Tab.~\ref{tab:mimiciv_scenario_tables_detailed}.
}%endcaption
\label{tab:mimiciv_results_main}
\end{table}
\definecolor{MyLightRed}{rgb}{0.99,0.45,0.45}

\begin{table}[!t]
\resizebox{0.48\textwidth}{!}{%
\begin{tabular}{l r r r r}
& \multicolumn{2}{c}{scenario 1} & {scenario 2}
\\
& \multicolumn{2}{c}{$\alpha = 0.15, \kappa = 0.5 |\mathcal{D}|$}
& \multicolumn{2}{c}{$\alpha = 0.5, \kappa = 0.1 |\mathcal{D}|$}
\\
& \multicolumn{2}{c}{$\frac{\cFP}{\cFN} \sim \text{Unif}(\frac{1}{9}, \frac{1}{6} )$}
& \multicolumn{2}{c}{$\frac{\cFP}{\cFN} \sim \text{Unif}(\frac{1}{40}, \frac{1}{20} )$}
%\\
%& $t \in [.5,.6]$
%& $t \in [.18,.31]$
\\
\hline
Strategy 
 & \makecell[b]{Avg. Cost \\ \small on test}
 & \makecell[b]{\scriptsize $\text{Prec} \geq \alpha$?\\ \scriptsize $\text{Cap} \leq \kappa$? \\ \small  val.~~test} 
 & \makecell[b]{Avg. Cost \\ \small on test}
 & \makecell[b]{\scriptsize $\text{Prec} \geq \alpha$?
 \\ \scriptsize $\text{Cap} \leq \kappa$? \\ 
 \small val.~~test} 
\\
\hline
max VOROS
& \textbf{0.319}~~ & Y~~~Y~~~
& 0.772~~ & Y~~~Y~~~
\\
max pAUROC
& 0.336~~
& \cellcolor{MyLightRed!30} Y~~~N~~~
& 0.707~~
& \cellcolor{MyLightRed!30} Y~~~N~~~ 
\\
max recall 
& 0.336~~
& \cellcolor{MyLightRed!30} Y~~~N~~~
& 0.707~~
& \cellcolor{MyLightRed!30} Y~~~N~~~
\\
max PV (ours)
& \textbf{0.319}~~ & Y~~~Y~~~
& 0.707~~
& \cellcolor{MyLightRed!30} Y~~~N~~~
\end{tabular}
}%end resizebox
\caption{
\textbf{eICU results for model selection.}
We report average costs on test (lower is better).
Strategies and scenarios are the same as for MIMIC-IV in Tab.~\ref{tab:mimiciv_results_main}.
All models satisfy constraints on validation, but some do not satisfy on test; further results in Tab.~\ref{tab:eicu_scenario_tables_detailed}.
}%endcaption
\label{tab:eICU_results_main}
\end{table}

\section{RELATED WORK}
\label{sec:related_work}
For an overview of performance metrics see~\citet{handAssessingPerformanceClassification2012}; for a focus on visuals, see~\citep{pratiSurveyGraphicalMethods2011}. 
%ROC curves remain a popular visual; see~\citet{streinerWhatsROCIntroduction2007} or \citet{fawcettIntroductionROCAnalysis2006} for accessible introductions. 
\citet{steyerbergBetterClinicalPrediction2014} provides advice for healthcare focused modeling, where beyond just assessing discrimination via ROC, calibration is also valuable.

A common summary of the ROC is the area under the curve~\citep{bradleyUseAreaROC1997,handAssessingPerformanceClassification2012}, known as AUROC or the concordance statistic.
%~\citep{hernandez-oralloUnifiedViewPerformance2012}.
Many works recommend a \emph{partial area} under the ROC by only integrating over some false positive rates~\citep{mcclishAnalyzingPortionROC1989,jiangReceiverOperatingCharacteristic1996,robinPROCOpensourcePackage2011}.
Our partial VOROS can be seen as a 3D extension of the partial AUROC that focuses on a desired cost range and obeys both capacity and precision constraints.
Other alternative partial area metrics seek better correspondence to concordance~\citep{carringtonNewConcordantPartial2020} but do not account for capacity or precision.

\citet{shao2024weighted}  extend ROC curves to cost-sensitive learning, seeking to make weighted area-under-curve training robust to train-to-test shifts in cost and class distribution. Their work does not address precision or capacity constraints, unlike our partial VOROS.

Particularly for hospital alert systems, some works recommend \emph{precision-recall} curves instead of (or in addition to) ROC curves~\citep{saitoPrecisionRecallPlotMore2015,romero-brufauWhyCstatisticNot2015,martinUseAreaPrecisionRecall2025}. However, claims that the PR curve or the area under it (AUPRC) is somehow superior to the ROC/AUROC for imbalanced data have been recently refuted~\citep{richardsonReceiverOperatingCharacteristic2024,mcdermottCloserLookAUROC2024}.
For more on PR curves, see \citet{flachPrecisionRecallGainCurvesPR2015}.

%\textbf{Training models to meet operational constraints.}
Synergistically with our work on a new performance metric, other work has sought to train models directly to perform well on certain metrics~\citep{tsoiBridgingGapUnifying2022,ebanScalableLearningNonDecomposable2017}. Some of these directly optimize for recall at a precision constraint~\citep{rath2022optimizing,fathonyAPperfIncorporatingGeneric2020,pengExactReformulationOptimization2025} or the area under the PR curve~\citep{ramziRobustDecomposableAverage2021}.
It may be possible to optimize for multi-objective criteria, such as maximizing recall and precision while minimizing capacity (and possibly other objectives) using a weighted sum of losses. 
This might be disfavored because it presents additional challenges like how to choose the weights on each objective term. Our approach that enforces hard constraints on precision or capacity better reflects the realities of intended applications: as long as the precision is high enough, cost alone rather than "cost $+$ precision" should be the driving objective.

\section{CONCLUSION}
\label{sec:conclusion}
We have developed the \textbf{partial volume over the ROC surface}, aka the partial VOROS or PV metric, as a performance measure for binary classifiers that accounts for cost-imbalance, precision constraints, and capacity constraints.
Our work represents careful geometric reasoning about tradeoffs between recall, precision, and false alarm rates.
Our experiments on real health records show how partial VOROS can help evaluate models and make model selection decisions that would meet necessary operational constraints and reduce costs when deployed.

\textbf{Limitations.}
We focus on binary classification; future work would be needed to transfer these ideas to a many-class or multi-label setting.
Computing this metric requires specific constraints $\alpha,\kappa$ to be specified as well as a distribution over costs, either explicitly via a distribution over fractional cost parameter $t$ or implicitly via a distribution on the ratio $\cRatio$.
It is likely challenging to elicit values for technical constraints from non-technical stakeholders, though some work suggests routes forward~\citep{wuNovelPartialArea2008}.

The PV metric is specialized to a particular class balance. Numerical PV values are not comparable across datasets with different prevalence. Naturally, due to the relationship between precision and prevalence, any performance metric that requires enforcing a precision constraint could not escape a sensitivity to prevalence.

Real applications often face \emph{drift} in class balance over time (e.g. due to emerging diseases, patient population shifts, or changing care practices). 
Stakeholder needs can evolve too, requiring different $\alpha$ or $\kappa$ values or different cost distributions over time.
Care is needed to wield PV effectively in such cases. 

Even without drift, ensuring constraint satisfaction on several datasets from the same source with the same prevalence can be challenging due to statistical variation, as Tab.~\ref{tab:eICU_results_main} indicates. In practice, enforcing a slightly higher $\alpha$ or lower $\kappa$ than strictly needed may better ensure generalization.
Future work could examine \emph{statistical guarantees} for constraint satisfaction. 

%\textbf{Outlook.}
%In clinical applications and beyond, the incorporation of precision and capacity constraints is critical to successful deployment. We hope our proposed metrics here help decision-makers find models that meet their needs.

\textbf{Recommendations.}
To help practitioners compute our PV metric on their data, we provide a User Guide in App.~\ref{supp:user_guide} and an open source release of Python code (see GitHub link on page 1).
We intend that PV can be used in at least two ways: (i) as a standalone diagnostic metric for model assessment by human inspection, and (ii) a metric for model selection. We provide recommendations for each use case in App.~\ref{supp:user_guide}.

\clearpage
\subsection*{Acknowledgements}
The authors acknowledge support from the U.S. National Science Foundation (NSF) via two awards: IIS \#2338962 and the 
DIAMONDS REU \#2149871.
A part of author MH's effort was also supported by the U.S. National Institutes of Health under award number R01HL180937. 
This paper's content is solely the responsibility of the authors and does not necessarily represent the official views of the NSF or NIH.

% MCH note to self:
% Carissa was supported by private donor, but still indirectly by CNS \#2149871.

\bibliographystyle{apalike}
\bibliography{references.bib,refs_from_zotero.bib}

\appendix
\onecolumn

\counterwithin{table}{section}
\setcounter{table}{0}
\counterwithin{figure}{section}
\setcounter{figure}{0}
\counterwithin{thm}{section}
\setcounter{thm}{0}

\section{User Guide for Partial VOROS}
\label{supp:user_guide}
In this section, we give a guide to effectively using the partial VOROS to evaluate classifier performance. The high-level outline is
\begin{itemize}
\item Step 1: Define scenario-specific settings: data settings $|\mathcal{P}|$, $|\mathcal{N}|$, constraints $\alpha,\kappa$, and costs $t$.
\item Step 2: Evaluate the partial VOROS.
\end{itemize}

\subsection{Step 0: Identify goals for model development and evaluation.}

Backing up, before calculating PV, we hope developers have a clear intention of what problem they want to solve.
We intend that PV can be
used in at least two ways: (i) as a standalone diagnostic
metric for model assessment by humans, and (ii) a metric for model selection.

First, as a standalone diagnostic, we hope that a human model developer could look at the relative PV values of two models and gain understanding about which one is delivering better binary classifications under the assumed constraints about $\alpha,\kappa$ and cost distribution $p(t)$. PV has a well-defined numerical range and directionality (perfect models get 1.0, higher is better).
In terms of single number metrics that summarize the possible confusion matrices of a classifier, we think PV should stand alone. We would suggest it be reported instead of alternative metrics like AUROC or AUPRC that assess discriminative quality via transformations of possible confusion matrices, but don't reflect cost or constraints. For example, showing PV and AUROC side-by-side doesn't help answer one question, but answers distinct questions with conflicting assumptions: AUROC does not reflect specific constraints or cost ranges, but PV does.
That said, PV may still be complementary to visualizations of the ROC curve. Showing these side-by-side can still be illuminating, especially when the constraint lines are overlaid. PV would also still be complementary to other metrics that assess classifier performance beyond the confusion matrix, such as metrics for calibration \citep{guoCalibrationModernNeural2017,vaicenaviciusEvaluatingModelCalibration2019}, group fairness  \citep{hardtEqualityOpportunitySupervised2016,pessachReviewFairnessML2022}, or net benefit~\citep{steyerbergBetterClinicalPrediction2014}.

Second, we do recommend that PV can be valuable in model selection. This requires several key settings to be known in advance, such as the deployment setting's constraints $\alpha,\kappa$ as well as a plausible cost range or distribution. We intend that Sec.~\ref{sec:mimiciv_experiments} provides a demonstration of how the PV performance metric can inform model and threshold selection, ultimately leading to lower cost predictions. It is critical to ensure all selection decisions align with the intended deployment or ``test'' scenario. In particular, the validation set must match deployment's expected prevalence and data distribution as much as possible. Even under well-matched circumstances, a selected model that barely meets constraints on validation may not meet them on test due to statistical noise. It may be wise to set $\alpha$ slightly higher and $\kappa$ slightly lower than strictly needed. Future work could examine generalization bounds for these constraints.

\subsection{Step 1: Define Scenario-specific Settings}

The first step is to define the values or distributions of various scalar quantities used to calculate the partial VOROS. These include:

%The first, and most important one is the relative sizes $\mathcal{P}$ and $\mathcal{N}$.
\begin{enumerate}
    \item \textbf{Define $\mathcal{P}$ and $\mathcal{N}$.} Determine the relative sizes of the positive-labeled set $\mathcal{P}$ and the negative-labeled set $\mathcal{N}$ which together comprise the overall dataset $\mathcal{D}$. These can be defined in absolute terms as $|\mathcal{P}|$ and $|\mathcal{N}|$, or in relative terms by knowing the total size $|\mathcal{D}|$ and the prevalence of the positive class $p=\frac{|\mathcal{P}|}{|\mathcal{D}|}$.
    It is important to remember that partial VOROS values are \textbf{only comparable on datasets of the same prevalence}. 
%While the partial VOROS can be calculated over a range of prevalence values $p=\frac{|\mathcal{P}|}{|\mathcal{D}|}$, for an apples-to-apples comparison, we assume a fixed prevalence which represents the likely relative size of the positive class.

    \item \textbf{Define $\alpha$.} The minimum precision constraint value $\alpha$ must be determined in consultation with appropriate stakeholders. This is a rate between 0.0 and 1.0. It represents the minimum allowable fraction of all alarms that must be correct for the system to be useful. 

    \item \textbf{Define $\kappa$.} The maximum capacity constraint value $\kappa$ can be specified as an absolute value taking any non-negative number on the scale of 0.0 to $|\mathcal{D}|$.
    We could also specify a fractional capacity $\tilde{\kappa}$ between 0.0 and 1.0, such that $\kappa = \tilde{\kappa} |\mathcal{D}|$, so the fraction represents capacity relative to total data size.
    In some cases, it may be more natural to specify $\kappa$ in absolute terms, such as when the clear limit is the total number of available staff.
    A hospital's capacity may be more readily interpreted as maximum number of alarms handled per day (or other unit of time), since a single hospital with finite staff might know exactly the number of alarms they can handle.
     However, different use cases may favor the fractional constraint. For example, an entire hospital system might prefer to specify a ratio instead.
     %as opposed to maximum percentage of patient alarms which can be handled per day, though which manner of specifying these is most intuitive may depend on the particulars of the analysis (a single hospital with a finite staff might know exactly the number of alarms they can handle, but an entire hospital system might prefer to specify a ratio instead).

     \item \textbf{Define cost.} The last two settings which must be specified, again in consultation with stakeholders, are the cost of a false positive $\cFP > 0$ and the cost of a false negative $\cFN > 0$.
     For the hospital alert use case, cost values should be specified relative to the usual standard of care. So the cost of a false negative is the cost of not giving additional immediate attention and treatment to a patient in need  (but still visiting them as planned in rounds). Similarly, the cost of a false positive should cover both the cost of treatment when none was needed and the opportunity cost of additional time spent caring for the patient beyond standard of care.
     \medskip
     
     \textbf{Cost as a Ratio.} The absolute values of these costs are not important, instead what matters is the ratio:  $\cRatio = \frac{\cFP}{\cFN}$.
    All absolute cost settings with the same ratio yield the same fractional cost $t$ via Def.~\ref{defn:t}. 
    \medskip
    
    \textbf{Uncertain costs.}
    A key strength of our PV approach is the ability to \emph{average over uncertainty} about the fractional cost $t \in [0, 1]$ or the cost ratio $\cRatio > 0$, rather than force a specific value. Users can specify a valid distribution over $t$ or directly over $\cRatio$. Using the one-to-one mapping between $t$ and $\cFP,\cFN$ values in Def.~\ref{defn:t}, we can write the expected value definition of PV in terms of whichever distribution is convenient, as shown below.
     % What is important in specifying these is their relative size, so these can be given as individual ranges or as a range on either the ratio $C_{FP}:C_{FN}$ or $C_{FN}:C_{FP}$.

\end{enumerate} 
\begin{align}
    PV(\mathcal{F}) &= 
        \mathbb{E}_{p(t)} [
        \max_{\tau} A^*_t( \mathcal{F}_{\tau} )
        ]
        \\ \notag 
    &= \mathbb{E}_{p(\cRatio)}[
        \max_{\tau} A^*_{t(\cRatio)}( \mathcal{F}_{\tau} )
        ], \qquad t(\cRatio) := \frac{\cRatio |\mathcal{N}|}{\cRatio |\mathcal{N}| + |\mathcal{P}|},
\end{align}
In practice, a direct distribution over the fractional cost parameter $t$ is complicated to select because it requires beliefs about $\cRatio$ as well as specific values related to the relative sizes of $\mathcal{P}$ and $\mathcal{N}$ in the dataset. Unless the composition of validation and test sets are exactly matched, the desired $p(t)$ will need to \emph{change} even across validation and test sets from the same data source.

We thus recommend directly specifying a distribution over $\cRatio$ as both more natural to stakeholders and transportable across datasets with different prevalence. See our analysis of Scenario 1 and Scenario 2 in Sec.~\ref{sec:mimiciv_experiments} for two practical versions of a $\cRatio$ distributions in our hospital alert application.

\subsection{Step 2: Evaluate the Partial VOROS}

Once all of the settings from Step 1 are appropriately specified, we can evaluate the partial VOROS metric. 

Exact evaluation of the Partial VOROS is only possible when the chosen distribution over $\cRatio$ (or $t$) has a PDF we can readily evaluate. 
We use simple uniform distributions throughout our experiments. 
Alternatively, if sampling is possible then Monte Carlo estimation of the integral could be done.

To evaluate PV in practice, we first use the constants $\alpha$, $\kappa$, $|\mathcal{P}|$, and $|N|$ to define the feasible region as a subset of ROC space (this is the intersection of six half spaces). 

Then, we obtain a dense grid of $\cRatio$ values and map each element to its corresponding $t$ value via Def.~\ref{defn:t}, yield a grid of $t$ values $\{T_1,...,T_N\}$. For each, $T_i$, pick the optimal allowed point on the ROC curve for $\mathcal{F}$. Note, such an optimal point can be found readily by looking at the slope of the observed ROC curve to the left, $m_{\text{left}}(x,f(x))$ and right $m_{\text{right}}(f,f(x))$ of any point $(x,f(x))$ on the curve in the feasible region, the optimal point will be the one where either $m_{\text{right}}\leq \frac{T_i}{1-T_i} \leq m_{\text{left}}$ or, if the optimal point on the curve lies outside the feasible region, the optimal point will be the rightmost feasible point on the curve. Each such point is associated a threshold $\tau_i$ for our classifier $\mathcal{F}$. When using the partial VOROS for model development and selection, the choice of this $\tau_i$ for the given $T_i$ should be determined on the validation set and then that same $\tau_i$ used on the test set.

Once we have a threshold $\tau_i$ associated to each $T_i$, we can calculate the partial area of lesser classifiers at that point. This is the area of the polygon below the isoperformance line associated to $\tau_i$ in the feasible region. Our Python code performs the required geometry and uses the shoelace formula to evaluate the area.

We now have the required ingredients to evaluate the integral that defines PV. 
Our experiments use a numerical approximation via the trapezoid method with 1000 evenly-space points covering the support of the uniform distribution over $\cRatio$. This yields the partial VOROS for the given $\alpha$, $\kappa$, $|\mathcal{P}|$ and $|\mathcal{N}|$.

Practitioners may consider task-specific ranges of the limit parameters $\alpha$, $\kappa$. The experiments in Fig.~\ref{fig:results_edelson} provide a reasonable guide. Remember that PV values can only be fairly compared for datasets with equal prevalence $p$.

 \section{Cases for Feasible Region Polygons}

 In this section, we provide a detailed analysis of the feasible region as well as explicit formulas for the vertices and areas of the feasible region.
 
 \label{supp_casesForPolygons}
 Using the notation of Definition 14, we have the following formulas for the potential vertices of the feasible region.

\begin{lem}[Coordinates of Vertices]
If $i,j \in \{0,1\}$, then $v_{ij}=(i,j)$, also if $i=0$, then $v_{i\alpha}=v_{\alpha i}=(0,0)$. Otherwise, we have
\begin{itemize}
\item $v_{\alpha \kappa} = \left(\dfrac{(1-\alpha)\kappa}{|\mathcal{N}|}, 
        \dfrac{\alpha \kappa}{|\mathcal{P}|}\right)$

\item $v_{\alpha 1}= \left( \dfrac{(1-\alpha)|\mathcal{P}|}{\alpha |\mathcal{N}|},1 \right)$

\item $v_{\kappa 1}=\left( \dfrac{\kappa - |\mathcal{P}|}{|\mathcal{N}|},1 \right)$

\item $v_{0 \kappa}=\left(0, \dfrac{\kappa}{|\mathcal{P}|} \right)$

\item $v_{\kappa 0}=\left( \dfrac{\kappa}{|\mathcal{N}|},0\right)$

\item $v_{1 \kappa}=\left(1,\dfrac{\kappa-|\mathcal{N}|}{|\mathcal{P}|}\right)$

\item $v_{1 \alpha}=\left(1, \dfrac{\alpha |\mathcal{N}|}{(1-\alpha)|\mathcal{P}|}\right)$
\end{itemize}

\end{lem}

This lemma follows directly from simple algebra using the definition.

In this section, we explain how to handle cases that don't fit into Definition~\ref{def:practical_assumptions}.

% Notation below is redundant as it ended up in the main paper.

%\begin{defn}[Notation for vertices]
%Let $i,j \in \{0,1\}$ and let $\beta \in \{\alpha, \kappa\}$. Define the following 9 vertices for all possible $i,j,\beta$

%\begin{enumerate}
%\item $v_{ij}$ is the intersection of $x=i$ and $y=j$
%\item $v_{\beta j}$ is the intersection of $\ell_\beta$ and $y=j$
%\item $v_{i\beta}$ is the intersection of $x=i$ and $\ell_\beta$.
%\item $v_{\alpha \kappa}$ is the intersection of $\ell_\alpha$ and $\ell_\kappa$.
%\end{enumerate}
%\end{defn}

Note: since $\ell_\alpha$ passes through the origin, we have $v_{00}=v_{\alpha 0}=v_{0\alpha}=(0,0)$ for all $\alpha$, so there are really only 7 points of interest for our analysis.

With a precision $\alpha \geq 0$ and maximum capacity $\kappa \geq 0$, there are 13 cases for our feasible region, we break these into $8$ degenerate and $5$ nondegenerate cases, where the 3 cases that satisfy Definition 11 are 3 of the 5 non-degenerate cases, and the other two non-degenerate cases violate Definition 11 because $\alpha < p$.  Let $\ell_{\alpha}$ denote the precision line and $\ell_{\kappa}$ denote the capacity line.

\subsection{Degenerate Cases}

First, note that if $\alpha=1$ or $\kappa=0$, then the feasible region is merely the point $(0,1)$ or $(0,0)$ respectively. Also, the feasible region is empty if $\alpha=1$ and $\kappa \leq |\mathcal{D}|$.

If $\alpha=0$ or $\kappa \geq |\mathcal{D}|$, we consider the feasible region to be degenerate as only one of the two bounds intersects the interior of ROC space.

The degenerate cases depend on how $\ell_{\alpha}$ or $\ell_{\kappa}$ intersect ROC space. If all we have is a minimum precision bound, there are two possibilities. $v_{\alpha 1}$ borders ROC space, or $v_{1 \alpha}$ borders ROC space. Note the first case is equivalent to nondegenerate case 1 below. Similarly, if all we have is a capacity bound, there are four cases depending on which pair of $v_{\kappa 0},v_{\kappa 1},v_{0\kappa}$ and $v_{1\kappa}$ border ROC space (note $v_{0 \kappa}$ and $v_{\kappa 1}$ cannot both border ROC space since $\ell_{\kappa}$ has nonpositive slope).

\subsection{Nondegenerate cases}

Assume that $\alpha \in (0,1)$ and $\kappa \in (0, |\mathcal{D}|)$, then we have five cases.

Case 1: $v_{\alpha \kappa}$ lies inside of ROC space.

This also splits into 2 subcases depending on whether $v_{01}$ is feasible according to $\ell_{\kappa}$.

 Case 1A: If $v_{01}$ is not feasible then the feasible region is the triangle $\triangle v_{00}v_{\alpha \kappa}v_{0\kappa}$. This is Case 1 in the main paper. 

Case 1B: If $v_{01}$ is feasible, then the feasible region is the quadrilateral $v_{00}v_{\alpha \kappa}v_{\kappa 1}v_{01}$. This is Case 2 in the main paper.

Case 2: $v_{\alpha \kappa}$ lies above $y=1$. (This is Case 3 in the main paper) 

In this case, the feasible region is simply the triangle $\triangle v_{00} v_{\alpha1}v_{01}$.

Case 3: $v_{\alpha \kappa}$ lies to the right of $x=1$. (Not in the main paper, violates Definition 11 since $\alpha < p$)

This splits into 2 subcases depending on whether $v_{01}$ is feasible according to $\ell_{\kappa}$.

Case 3A: if $v_{01}$ is feasible, then the feasible region is the pentagon $v_{00},v_{1\alpha}v_{1\kappa}v_{\kappa 1}v_{01}$

Case 3B: if $v_{01}$ is not feasible, then the feasible region is the trapezoid $v_{00}v_{1\alpha}v_{1\kappa}v_{0\kappa}$.

\subsection{Geometry of Polygon Cases Satisfying Definition 11}
\label{supp:geometry_cases_satisfying}

In this section we detail the form and area of the feasible region using the notation of Definition~\ref{def:vertices} subject to the cases in Section 3.4 (and also denoted Cases A, C1 and C2 above): 

\begin{lem}[Vertices for Feasible Regions]
Counterclockwise from the origin, the bounding vertices of the Case 1 Triangle are
    \begin{itemize}
        \item $v_{00} = (0,0)$
        \item $v_{\alpha \kappa} = (\frac{(1-\alpha)\kappa}{|\mathcal{N}|}, 
        \frac{\alpha \kappa}{|\mathcal{P}|})$
        \item $v_{0\kappa } = (0, \frac{\kappa}{|\mathcal{P}|})$
    \end{itemize}

Counterclockwise from the origin, the bounding vertices of the Case 2 Quadrilateral are
\begin{enumerate}[noitemsep]
\item $v_{00}=(0,0)$
\item $v_{\alpha\kappa} =\left(\dfrac{(1-\alpha)\kappa}{|\mathcal{N}|},\dfrac{\alpha \kappa}{|\mathcal{P}|}\right)$
\item $v_{\kappa 1}=\left(\dfrac{\kappa-|\mathcal{P}|}{|\mathcal{N}|},1\right)$
\item $v_{01}=(0,1)$
\end{enumerate}

Counterclockwise from the origin, the bounding vertices of the Case 3 triangle are

\begin{enumerate}[noitemsep]
\item $v_{00}=(0,0)$
\item $v_{\alpha 1}=\left(\dfrac{(1-\alpha)|\mathcal{P}|}{\alpha |\mathcal{N}|},1\right)$
\item $v_{01}=(0,1)$
\end{enumerate}
\label{lem:explicit_coordinates}
\end{lem}

Each of these vertices can be found by simple algebra.

In order to decide which of the 3 cases to use in the area formula given by \ref{lem:formula_for_partial_area}, we can use the following lemma.

\begin{lem}[Area of Feasible Region]
\label{lem:area_feasible_region}
In general, the area of the feasible region in ROC space is

$$A^*=\begin{cases}
\dfrac{(1-\alpha)\kappa^2}{2|\mathcal{N}||\mathcal{P}|} & \text{case 1}\\
\dfrac{2 \kappa |\mathcal{P}|-\alpha\kappa^2-|\mathcal{P}|^2}{2|\mathcal{N}|\mathcal{P}|} & \text{case 2}\\
\dfrac{(1-\alpha)|\mathcal{P}|}{\alpha|\mathcal{N}|} & \text{case 3}
\end{cases}$$

%$$A = \dfrac{|\mathcal{P}|(1-\alpha)\kappa+(\kappa-|\mathcal{P}|)(|\mathcal{P}|-\alpha \kappa)}{2 |\mathcal{N}| |\mathcal{P}|}$$

\end{lem}

\begin{proof}
This follows by applying the well-known ``shoelace'' formula~\citep{leeShoelaceFormulaConnecting2017,zwillingerCRCPolygonFormula2018} for the area of a polygon applied to the boundary vertices defined above.
\end{proof}

This lemma shows that the area of the feasible region can be calculated precisely from the constraints defining the problem and dataset. Alternatively, since the feasible region consists of the convex hull of 6 lines in the plane, software can calculate the area of the convex hull of the collection of 6 lines $x=0$, $x=1$, $y=0$, $y=1$, $\ell_\kappa$, and $\ell_\alpha$, quite quickly (see Section~\ref{supp_compComplexity} below).

 \section{Cases for Partial Areas}
 \label{supp_isoperformanceCases}
 Given that we are in one of the three cases from Sec.~\ref{sec:geom_feasible_region}, there are four cases for the region of lesser classifiers.

Assuming the practical assumptions in Definition~\ref{def:practical_assumptions} we have the following Lemma.

\parbox{0.98\linewidth}{
\begin{lem}
\label{lem:never_alarm}
The isoperformance line through the baseline of never alarming does not intersect the interior of the feasible region iff $t \leq \dfrac{\alpha | \mathcal{N}|}{\alpha|\mathcal{N}|+(1-\alpha)|\mathcal{P}|}$. 
\end{lem}
}
\begin{proof}
The isoperformance line through $(0,0)$ is $y=\dfrac{t}{1-t}x$ which lies below $\ell_{\alpha}$ in the first quadrant precisely when $t < \dfrac{\alpha |\mathcal{N}|}{\alpha|\mathcal{N}|+(1-\alpha)|\mathcal{P}|}$. If equality holds, the lines coincide, but $\ell_t$ still misses the interior.
\end{proof}

This lemma means that the baseline of never alarming has the highest cost of any point in our feasible region, it also ensures that the partial VOROS of this baseline is $0$.

Using the notation of definition~\ref{def:vertices} we have

\begin{lem}[Coordinates for Iso-performance Vertices.]

We have the following

\begin{itemize}
\item $v_{0t}=(0,k-\frac{t}{1-t} h)$.
\item $v_{t1}=(\frac{(1-t)(1-k)}{t}+h,1)$
\item $v_{\alpha t}$ has $x$-coordinate given by $A'+\dfrac{B'}{C't-D'}$ where 
\begin{align*}
A'&=\dfrac{(1-\alpha)|\mathcal{P}|(h+k)}{\alpha |\mathcal{N}|+(1-\alpha)|\mathcal{P}|}\\
B'&=(1-\alpha)|\mathcal{P}|\left(\dfrac{\alpha |\mathcal{N}|h-(1-\alpha)|\mathcal{P}|k}{\alpha |\mathcal{N}|+(1-\alpha)|\mathcal{P}|} \right)\\
C'&=\alpha |\mathcal{N}|+(1-\alpha)|\mathcal{P}|\\
D'&=\alpha |\mathcal{N}|
\end{align*}
\item $v_{\kappa t}$ has $x$-coordinate $A+\dfrac{B}{Ct+D}$ where 
\begin{align*}
A&=\dfrac{|\mathcal{P}|(h+k)-\kappa}{|\mathcal{P}|-|\mathcal{N}|}\\
B&=\kappa-k|\mathcal{P}|-\dfrac{|\mathcal{N}|(|\mathcal{P}|(h+k)-\kappa)}{|\mathcal{P}-|\mathcal{N}|}\\
C&=|\mathcal{P}|-|\mathcal{N}|\\
D&=|\mathcal{N}|
\end{align*}
\end{itemize}
\end{lem}

\parbox{0.98\linewidth}{
\begin{proof}
The proof follows from taking the intersection of the lines $\ell_t$, $\ell_\alpha$, $\ell_{\kappa}$, $y=1$ and $x=0$ as needed. The rational linear $t$-terms in the $x$-coordinates of $v_{\alpha t}$ and $v_{\kappa t}$ follow from the fact that the coefficient of $x$ in $\ell_t$ is rational in $t$.
\end{proof}
}

\begin{lem}
The region of lesser classifiers for a feasible $(h,k)$ is always one of: 

\begin{itemize}
\item The triangle $v_{00}v_{\alpha t}v_{0t}$.
\item The quadrilateral $v_{00}v_{\alpha \kappa}v_{\kappa t}v_{0t}$
\item The pentagon $v_{00}v_{\alpha \kappa}v_{\kappa 1}v_{t1}v_{0t}$.
\item The quadrilateral $v_{00}v_{\alpha 1}v_{t1}v_{0t}$.
\end{itemize}
Where, the triangle $v_{00}v_{\alpha t}v_{0t}$ can apply to any of the three cases in section 3., so long as $\ell_t$ intersects $\ell_\alpha$ below $\ell_{\kappa}$ and below $y=1$. The quadrilateral $v_{00}v_{\alpha \kappa}v_{\kappa t}v_{0t}$, can only apply to cases 1 and 2 from Sec.~\ref{sec:geom_feasible_region} since $\ell_t$ needs to intersect $\ell_\kappa$ on the border of the feasible region. The pentagon only applies when $\ell_t$ intersects $y=1$ on the border of the feasible region in case $2$ from Sec.~\ref{sec:geom_feasible_region}. Finally, the quadrilateral $v_{00}v_{\alpha 1}v_{t1}v_{0t}$ only applies when $\ell_t$ intersects $y=1$ on the border of the feasible region in case $3$.
\end{lem}

\begin{proof}
Since we are in one of the three cases from Sec.~\ref{sec:geom_feasible_region}, note that the isoperformance line will always leave the feasible region on the left through the $y$-axis by the last part of Definition 11, so $v_{0t}$ is a vertex of all the regions.

The remaining possible regions depend on how the line leaves to the right and which points are included in the polygon. Specifically, 

\begin{enumerate}
\item if we are in Case 1, the region of lesser classifiers is the triangle $v_{00}v_{\alpha t}v_{0t}$ if $\ell_t$ lies on or below $v_{\alpha \kappa}$, and is the quadrilateral $v_{00}v_{\alpha \kappa}v_{\kappa t} v_{0t}$ otherwise.
\item in Case 2, the region of lesser classifiers is the triangle $v_{00}v_{\alpha t} v_{0t}$ if $\ell_t$ lies on or below $v_{\alpha \kappa}$, the quadrilateral $v_{00} v_{\alpha \kappa}v_{\kappa t}v_{0t}$ if $\ell_t$ lies between $v_{\alpha \kappa}$ and $v_{\kappa 1}$, or the pentagon $v_{00}v_{\alpha \kappa}v_{\kappa 1}v_{t1}v_{0t}$ if $\ell_t$ lies above $v_{\kappa 1}$.
\item in Case 3, the region of lesser classifiers is the triangle $v_{00}v_{\alpha t}v_{0t}$ if $\ell_t$ lies on or below $v_{\alpha 1}$, or the quadrilateral $v_{00}v_{\alpha 1} v_{t1}v_{0t}$ otherwise.
\end{enumerate}
\end{proof}

It is worth noting that whether $\ell_t$ lies above or below any of these vertices can be readily checked by comparing the cost of $(h,k)$ to that of the boundary point (higher cost points lie below lower costs ones), so this combined with Lemma~\ref{lem:explicit_coordinates} yields a simple algorithm for determining the form of the partial area.

\begin{lem}[Calculating Partial Area]
Let $t \in [0,1]$ be fixed, and let $(h,k)$ be the ROC coordinates of an allowed classifier, then the (non-normalized) partial area of lesser classifiers is given by the formula in Figure 3. 

%This formula has 4 cases. 

%In the first case, the iso-performance segment runs from the $y$-axis to the minimum precision line $\ell_{\alpha}$.  In the second case, the iso-performance segment runs from the $y$-axis to the maximum capacity line $\ell_{\kappa}$; and in the third case, the iso-performance segment $\ell_t$ runs from the $y$-axis to the line $y=1$ and we are in case 2 from section 3.4. The fourth case occurs when $\ell_t$ runs from the $y$-axis to the line $y=1$ and we are in case 3 from section 3.4.

The formula for the first case relies on the $x$ coordinate of $v_{\alpha t}$, $x = A'+\dfrac{B'}{C't-D'}$, where we define
\begin{align*}
A'&= \dfrac{(1-\alpha)|\mathcal{P}|(h+k)}{\alpha|\mathcal{N}|+(1-\alpha)|\mathcal{P}|}\\
B' &= (1-\alpha)|\mathcal{P}| \left( \dfrac{(\alpha|\mathcal{N}|h-(1-\alpha)|\mathcal{P}|k}{\alpha |\mathcal{N}|+(1-\alpha)|\mathcal{P}|}\right)\\
C' &= \alpha |\mathcal{N}| +(1-\alpha) |\mathcal{P}|\\
D' &= \alpha |\mathcal{N}|.
\end{align*}

The formula of the second case relies on the $x$ coordinate of $v_{\kappa t}$, $x = A+\dfrac{B}{Ct+D}$, where
\begin{align*}
A&=\dfrac{|\mathcal{P}|(h+k)-\kappa}{|\mathcal{P}|-|\mathcal{N}|}\\
B&=\kappa - k|\mathcal{P}|-\dfrac{(|\mathcal{P}|(h+k)-\kappa)|\mathcal{N}|}{|\mathcal{P}|-|\mathcal{N}|}\\
C&=|\mathcal{P}|-|\mathcal{N}|\\
D&=|\mathcal{N}|
\end{align*}

\label{lem:formula_for_partial_area}
\end{lem}

\begin{proof}
The proof follows directly from applying the well-known determinant formula for a polygon. Each formula corresponds to a distinct polygon from the preceding lemma. Specifically:
\begin{itemize}[noitemsep,leftmargin=*]

\item In the first formula, we are calculating the area of triangle $v_{00}v_{\alpha t}v_{0t}$.
\item In the second formula, we are calculating the area of the quadrilateral $v_{00}v_{\alpha \kappa}v_{\kappa t}v_{0t}$.
\item In the third formula, we are calculating the area of the pentagon, $v_{00}v_{\alpha \kappa}v_{\kappa 1}v_{t1}v_{0t}$.
\item In the fourth formula we are calculating the area of the quadrilateral $v_{00} v_{\alpha 1}v_{t1}v_{0t}$
\end{itemize}

\end{proof}

 \section{Computational Complexity}
 \label{supp_compComplexity}
 
Given a classifier $\mathcal{F}$, we can produce an ROC curve by taking all possible binarized classifiers $\mathcal{F}_{\tau}$ and plotting them in ROC space. To distinguish which $\mathcal{F}_{\tau}$ are potentially useful, we can take the convex hull of the curve together with the point $(1,0)$ in $O(n\text{log}(h))$ time, where $h$ is the number of vertices in the Convex Hull using Chan's Algorithm \citep{chan1996optimal}.

Then, given the convex hull, we can associate to each feasible point a range of values of $t$ for which that point will have the lowest cost on the curve and hence the highest area of lesser classifiers.
%\newpage
\begin{algorithmic}
\State Let $\{(x_i,y_i)\}$ the points of the convex hull oriented clockwise from $(0,0)$ to $(1,1)$
\If{$(x_{i+1},y_{i+1})$ is feasible}
    \If{$x_i=0$} 
        \If{$x_{i+1}=0$} 
        \State Skip.
        \Else 
        \State $(x_i,y_i)$ is optimal for $t \in [\frac{y_{i+1}-y_i}{x_{i+1}-x_i+y_{i+1}-y_i},1]$
        \EndIf
    \Else 
        \If{$y_i = 1$}
        \State $(x_i,y_i)$ is optimal for $t$ in $\left[0,\frac{y_i-y_{i-1}}{x_i-x_{i+1}+y_i-y_{i+1}}\right]$; Break
        \Else
        \State $(x_i,y_i)$ is optimal for $t$ in $\left[\frac{y_{i+1}-y_i}{x_{i-1}-x_i+y_{i-1}-y_i},\frac{y_i-y_{i-1}}{x_i-x_{i-1}+y_i-y_{i-1}}\right]$
        \EndIf
    \EndIf
\Else 
    \If{$(x_i,y_i)$ is feasible.}
    \State $(x_i,y_i)$ is optimal for $t \in \left[0,\frac{y_i-y_{i-1}}{x_i-x_{i-1}+y_i-y_{i-1}}\right]$
    \State Break
    \EndIf
\EndIf 

The only computations in this algorithm are the values of $\frac{y_{i+1}-y_i}{x_{i+1}-x_i+y_{i+1}-y_i}$ for all but the last feasible $(x_i,y_i)$ and checks whether each $(x_i,y_i)$ is feasible. Together, this takes order $n_{\alpha\kappa}$ time where $n_{\alpha\kappa}$ is the number of feasible points on the convex hull.

From here, we can use the formula for partial area in Figure 3 to calculate the areas and average them in $O(n_{\alpha \kappa})$ time.

Since the formulas in Figure 3 are linear combinations of quotients of linear functions of $t$, we can also integrate them directly getting a formula for the Volume associated to each $t$-range involving a logarithm. Computing this is again $O(n_{\alpha \kappa})$, since for a given $t$-range $[a,b]$, we can partition $[a,b]$ into subintervals $[a_i,b_i]$ such that the optimal feasible point is $(x_i,f(x_i))$ on the interval $[a_i,b_i]$. Thus the integral can be broken up into $n_{\alpha \kappa}$ integrals each of which can be evaluated in $O(1)$ time. The association of $x_i$'s to $[a_i,b_i]$ is also $O(n_{\alpha \kappa})$.

This yields an overall computational complexity of $O(n_{\alpha\kappa} \: \text{log} \; n_{\alpha\kappa}))$ if we need to compute the ROC convex hull. The partial ROC convex hull takes only $O(n_{\alpha\kappa}\:\text{log}\;n_{\alpha\kappa})$ or $O(n_{\alpha\kappa})$ time since we can start with $(0,0)$ and proceed right along the ROC curve until we hit the left boundary of the feasible region. If we already have the convex hull to begin with, we need only consider the points of the hull in the feasible region a determination that takes only $O(n_{\alpha\kappa})$ steps and from which the partial VOROS can be computed exactly in $O(n_{\alpha\kappa})$ steps.
\end{algorithmic}

Alternatively, we can approximate the partial VOROS by drawing a monte carlo sample of $p(t)$ and evaluating the area of lesser classifiers of our curve in $O(n_{\alpha \kappa})$, since we can find the optimal feasible point in an $O(\text{log}(n_{\alpha \kappa})$ step binary search, and then for each sample we can calculate the area of lesser classifiers in $O(1)$ steps. These partial areas can then be averaged to give a good approximation of the partial VOROS.

 \section{Details on Experiments}
 \label{supp_mimiciv}
 \subsection{License and availability}

Our experiments in Sec.~\ref{sec:mimiciv_experiments} use two open-access datasets available via PhysioNet~\citep{goldberger2000physionet} at \url{https://physionet.org/}.
First, the MIMIC-IV dataset \citep{johnsonMIMICIVFreelyAccessible2023} is freely available to qualified researchers subject to the PhysioNet Credentialed Health Data License 1.5.0.
Second, we used version 2.0 of the eICU Collaborative Research Database~\citep{pollardEICUCollaborativeResearch2018,PhysioNet-eicu-crd-2.0}, which is also available under the PhysioNet Credentialed Health Data License 1.5.0.

% We further release a few assets of our own:
% \begin{itemize}
% \item our own code is released under the MIT License (see link on page 1 of this supplement)
% \item our extracted ROC curves, which we obtained by carefully tracing the ROC plots in the published paper by \citet{edelsonEarlyWarningScores2024}, are also released under this license and available in our repository (link on page 1 of this supplement)
% \end{itemize}

\subsection{MIMIC-IV Mortality Prediction: Features and Task Setup}

For the mortality prediction experiments on MIMIC-IV data, we examine 6 vital signs, collected via bedside monitors and extracted from health records via the CHARTEVENTS table in MIMIC-IV.
We further examine 7 laboratory measurements from extracted blood and other fluids, again from the CHARTEVENTS table in MIMIC-IV.
See listing in Tab.~\ref{tab:vitals_labs}

These vitals and labs are extracted via best practices in \emph{clinical grouping} of conceptually similar variables that have distinct ITEMID codes in the EHR. We use the groupings provided in~\citet{wangMIMICExtractDataExtraction2020}, given explicitly in Tab.~\ref{tab:big_list_of_itemids} for our variables of interest.
Note that we extracted weight from the charts over time, but treated it as static (not dynamic) due to the limited 48 hour window.

\begin{table*}[!t]
\begin{minipage}{0.4\textwidth}
\begin{center}
    VITALS
\end{center}
\resizebox{\textwidth}{!}{
\begin{tabular}{l l}
Vital Sign & \texttt{feat\_name} in code
\\
\hline 
Blood pressure (diastolic) & bp\_diastolic\_mmHg
\\
Blood pressure (systolic) & bp\_systolic\_mmHg
\\
Heart rate & heart\_rate
\\
Oxygen saturation & oxygen\_saturation
\\
Respiratory Rate & resp\_rate
\\
Temperature & temp
\end{tabular}
}%endresizebox
\end{minipage}
\begin{minipage}{0.4\textwidth}
\begin{center}
LAB MEASUREMENTS
\end{center}
\resizebox{\textwidth}{!}{
\begin{tabular}{l l}
Lab Measurement & \texttt{feat\_name} in code
\\ \hline
    Cholesterol & cholesterol
    \\
    Glucose & glucose
    \\
    Hemoglobin & hemoglobin
    \\
    Lactic Acid & lactic\_acid
    \\
    pH & pH
    \\
    platelets & platelets
    \\
    white blood cell count & white\_blood\_cell\_count
    \end{tabular}
}%end resizebox
\end{minipage}
\caption{Summary of 6 vitals and 7 labs used in MIMIC-IV experiments}
\label{tab:vitals_labs}
\end{table*}

\begin{table*}[!t]
\begin{tabular}{r l r r l}
itemid & label                                 & unitname & feat\_name                &  \\
224643 & Manual Blood Pressure Diastolic Left  & mmHg     & bp\_diastolic\_mmHg       &  \\
225310 & ART BP Diastolic                      & mmHg     & bp\_diastolic\_mmHg       &  \\
220180 & Non Invasive Blood Pressure diastolic & mmHg     & bp\_diastolic\_mmHg       &  \\
220051 & Arterial Blood Pressure diastolic     & mmHg     & bp\_diastolic\_mmHg       &  \\
227243 & Manual Blood Pressure Systolic Right  & mmHg     & bp\_systolic\_mmHg        &  \\
224167 & Manual Blood Pressure Systolic Left   & mmHg     & bp\_systolic\_mmHg        &  \\
220179 & Non Invasive Blood Pressure systolic  & mmHg     & bp\_systolic\_mmHg        &  \\
225309 & ART BP Systolic                       & mmHg     & bp\_systolic\_mmHg        &  \\
220050 & Arterial Blood Pressure systolic      & mmHg     & bp\_systolic\_mmHg        &  \\
220603 & Cholesterol                           &          & cholesterol               &  \\
220621 & Glucose (serum)            &  & glucose                   &  \\
226537 & Glucose (whole blood)      &  & glucose                   &  \\
220045 & Heart Rate                            & bpm      & heart\_rate               &  \\
226730 & Height (cm)                           & cm       & height                    &  \\
226707 & Height                                & Inch     & height                    &  \\
220228 & Hemoglobin                            & g/dl     & hemoglobin                &  \\
225668 & Lactic Acid                           &          & lactic\_acid              &  \\
220277 & O2 saturation pulseoxymetry           & \%       & oxygen\_saturation        &  \\
220227 & Arterial O2 Saturation                & \%       & oxygen\_saturation        &  \\
223830 & PH (Arterial)                         &          & pH                        &  \\
220274 & PH (Venous)                           &          & pH                        &  \\
227457 & Platelet Count                        &          & platelets                 &  \\
224422 & Spont RR                              & bpm      & resp\_rate                &  \\
220210 & Respiratory Rate                      & insp/min & resp\_rate                &  \\
224689 & Respiratory Rate (spontaneous)        & insp/min & resp\_rate                &  \\
224690 & Respiratory Rate (Total)              & insp/min & resp\_rate                &  \\
223762 & Temperature Celsius                   & °C       & temp                      &  \\
223761 & Temperature Fahrenheit                & °F       & temp                      &  \\
224639 & Daily Weight                          & kg       & weight                    &  \\
226512 & Admission Weight (Kg)                 & kg       & weight                    &  \\
220546 & WBC                                   &          & white\_blood\_cell\_count & 
\end{tabular}
\caption{
    Exact ITEMIDs extracted from CHARTEVENTS table in MIMIC-IV
}
\label{tab:big_list_of_itemids}
\end{table*}

\clearpage 

\subsection{eICU Mortality Prediction: Features and Task Setup}

\textbf{Task description.}
For each ICU-stay in the eICU dataset \citep{PhysioNet-eicu-crd-2.0}, we collect static information upon admission and time-dependent health data for the first 48 hours. The classification task is to predict a patient’s mortality during their stay---if the patient will survive (class 0) or die (class 1) after the first 48 hours. We use the preprocessed eICU dataset extracted by FIDDLE consisting of binarized time-invariant and time-dependent features for each hour \citep{fiddle_extract, PhysioNet-mimic-eicu-fiddle-feature-1.0.0}. We reformatted the extract features to represent values for 2-hour buckets.

The train/valid./test data sets contain 38532/19267/19267 ICU-stays with prevalence $p$ of 0.116/0.114/0.114 (fraction of death cases). 

\textbf{Static Features.}
For static features, we chose 5 time-invariant variables, similar to those used in the MIMIC-IV experiment, which were extracted to 24 discretized and/or one-hot encoded features (Table \ref{tab:static-variables}).

\begin{table}[!h]
\centering

\begin{tabular}{ll}
\hline
\textbf{Variable Name} & \textbf{Value Ranges} \\
\hline
admissionheight &
$(-0.001,160.0]$ \\
& $(160.0,167.0]$ \\
& $(167.0,172.7]$ \\
& $(172.7,180.0]$ \\
& $(180.0,612.6]$ \\[0.5em]

admissionweight &
$(-0.001,63.5]$ \\
& $(63.5,74.8]$ \\
& $(74.8,86.2]$ \\
& $(86.2,102.5]$ \\
& $(102.5,953.0]$ \\[0.5em]

age &
$(-0.001,51.0]$ \\
& $(51.0,61.0]$ \\
& $(61.0,69.0]$ \\
& $(69.0,78.0]$ \\
& $(78.0,89.0]$ \\
& $>89$ \\[0.5em]

ethnicity &
African American \\
& Asian \\
& Caucasian \\
& Hispanic \\
& Other / Unknown \\[0.5em]

gender &
Female \\
& Male \\
\hline

\end{tabular}
\caption{Static Features and Associated Value Ranges used in eICU experiments}
\label{tab:static-variables}
\end{table}

\textbf{Time-varying features.}
We apply 6 summary functions (mean, variance, min, median, max, slope) to each of the 85 FIDDLE features, comprising of 17 unique time-varying variables which were discretized and one-hot encoded (Table \ref{tab:dynamic-variables-detailed}), over 3 windows (0-48 hours, 24-48 hours, 32-48 hours). We chose variables that also had associated minimum and maximum features in the complete extracted features. This meant that during extraction, these were deemed ``frequent" variables. For example, a feature that was included indicates if temperature measured was between the range $(36.4,37.0$]. We removed the ``was ever measured'' and ``time since last measurement'' summary functions used with MIMIC-IV, as the extracted features from FIDDLE already removed missing values. A future step could be to incorporate the ``mask" variable (which likely represents missingness) found in the extracted features from FIDDLE post-application of summary functions, as we do not want to apply the other 6 summary functions to ``mask."

\begin{table}[!h]
\centering
\begin{tabular}{ll}
\hline
\textbf{Variable Name} & \textbf{Value Range} \\
\hline

Vital Signs\textbar{}Heart Rate\textbar{}Heart Rate &
$(-0.001, 70.0]$, $(70.0, 80.0]$, $(80.0, 90.0]$, \\
& $(90.0, 102.0]$, $(102.0, 7734.0]$ \\[0.5em]

Vital Signs\textbar{}O2 Saturation\textbar{}O2 Saturation &
$(-0.001, 95.0]$, $(95.0, 96.0]$, \\
& $(96.0, 98.0]$, $(98.0, 100.0]$ \\[0.5em]

Vital Signs\textbar{}Respiratory Rate\textbar{}Respiratory Rate &
$(-0.001, 15.0]$, $(15.0, 18.0]$, $(18.0, 20.0]$, \\
& $(20.0, 24.0]$, $(24.0, 330.0]$ \\[0.5em]

cvp &
$(-95.001, 6.0]$, $(6.0, 10.0]$, $(10.0, 14.0]$, \\
& $(14.0, 21.0]$, $(21.0, 400.0]$ \\[0.5em]

heartrate &
$(-0.001, 70.0]$, $(70.0, 80.0]$, $(80.0, 90.0]$, \\
& $(90.0, 102.0]$, $(102.0, 300.0]$ \\[0.5em]

noninvasivediastolic &
$(-0.001, 53.0]$, $(53.0, 60.0]$, $(60.0, 67.0]$, \\
& $(67.0, 77.0]$, $(77.0, 263.0]$ \\[0.5em]

noninvasivemean &
$(-0.001, 67.0]$, $(67.0, 75.0]$, $(75.0, 84.0]$, \\
& $(84.0, 94.0]$, $(94.0, 285.0]$ \\[0.5em]

noninvasivesystolic &
$(-0.001, 101.0]$, $(101.0, 113.0]$, $(113.0, 125.0]$, \\
& $(125.0, 140.0]$, $(140.0, 287.0]$ \\[0.5em]

respiration &
$(-0.001, 15.0]$, $(15.0, 18.0]$, $(18.0, 21.0]$, \\
& $(21.0, 24.0]$, $(24.0, 194.0]$ \\[0.5em]

sao2 &
$(-0.001, 94.0]$, $(94.0, 96.0]$, \\
& $(96.0, 98.0]$, $(98.0, 100.0]$ \\[0.5em]

st1 &
$(-25.001, -0.2]$, $(-0.2, -0.01]$, $(-0.01, 0.0]$, \\
& $(0.0, 0.2]$, $(0.2, 1090.0]$ \\[0.5em]

st2 &
$(-16.551, -0.4]$, $(-0.4, -0.09]$, $(-0.09, 0.1]$, \\
& $(0.1, 0.48]$, $(0.48, 1170.0]$ \\[0.5em]

st3 &
$(-23.401, -0.3]$, $(-0.3, -0.08]$, $(-0.08, 0.1]$, \\
& $(0.1, 0.4]$, $(0.4, 1170.0]$ \\[0.5em]

systemicdiastolic &
$(-75.001, 48.0]$, $(48.0, 55.0]$, $(55.0, 62.0]$, \\
& $(62.0, 71.0]$, $(71.0, 395.0]$ \\[0.5em]

systemicmean &
$(-72.001, 66.0]$, $(66.0, 74.0]$, $(74.0, 81.0]$, \\
& $(81.0, 93.0]$, $(93.0, 397.0]$ \\[0.5em]

systemicsystolic &
$(-94.001, 98.0]$, $(98.0, 112.0]$, $(112.0, 124.0]$, \\
& $(124.0, 141.0]$, $(141.0, 398.0]$ \\[0.5em]

temperature &
$(-2622.401, 36.4]$, $(36.4, 37.0]$, $(37.0, 37.4]$, \\
& $(37.4, 37.8]$, $(37.8, 105.8]$ \\

\hline
\end{tabular}
\caption{Dynamic Features and Associated Value Ranges used in eICU experiments}
\label{tab:dynamic-variables-detailed}
\end{table}

% \\\\
% \textbf{Classifiers.}
% We use logistic regression, multi-layer perceptron, and random forest as classifiers, grid searching over 48+ hyperparameter configurations (the same configurations used in the MIMIC-IV experiment).

% \textbf{Evaluation Plan.}
% We compare model selection using 2 different strategies—maximum recall in feasible region and maximum partial VOROS. For each selection strategy, we select the best hyperparameter configuration and the best threshold $\tau(t)$ for a given $t$ range. For the cost-aware strategy partial VOROS, the threshold depends on the $t$ in the range, while for the cost-unaware strategy maximum recall, $\tau(t)$ produces the same threshold (we use the threshold that maximizes recall under precision and capacity constraints). We then use each selected model and threshold for both cost parameter ranges to make predictions on the test set (unseen data), then calculate the expected cost over each cost parameter range through Monte-Carlo sampling. 
% \\\\
\textbf{Evaluation plan.}
Given the similarity of the eICU mortality prediction task with the task from MIMIC-IV, we use the same scenarios for setting $\alpha, \kappa, \cRatio$, the same list of possible classifiers and hyperparameters, and the same selection strategies.

In Scenario 1 (more permissive), on eICU data the chosen cost ratio distribution $\cRatio \sim \text{Unif}(\frac{1}{9},\frac{1}{6})$ maps to a non-uniform distribution over the fractional cost parameter $t\in[0.46,0.56]$.

In Scenario 2 (more constrained), on eICU data the chosen cost ratio distribution $\cRatio \sim \text{Unif}(\frac{1}{40},\frac{1}{20})$ maps to a non-uniform distribution over $t \in [0.16, 0.28]$.

% We look at two scenarios of precision, capacity, and misclassification costs. In scenario 1, we use $\alpha=0.15,\kappa=0.5|D|,\frac{C_0}{C_1}=\text{Unif}(\frac{1}{9},\frac{1}{6})$ (where $C_0$ is the cost of a false positive and $C_1$ is the cost of a false negative). This cost ratio maps to cost parameter $t\in[0.46,0.56]$. In mortality prediction, false negative should always be more costly than a false positive (not alerting for an at-risk patient who may then die has greater cost than an unnecessary alert), but here false negatives cost less than 10x the cost of a false positive. In scenario 2, we use $\alpha=0.5,\kappa=0.1|D|,\frac{C_0}{C_1}=\text{Unif}(\frac{1}{40},\frac{1}{20})$ to match economic estimates of costs where false negative is 20x to 40x greater cost than a false positive \citep{rogers_optimizing_2023}. This maps to $t\in[0.16,0.28]$.

\clearpage

\subsection{Model Development, Hyperparameter Grid, and Computational Hardware}
\label{supp:model_dev_and_hypers}

Given a specific dataset for the mortality prediction task (either MIMIC-IV or eICU), we train many versions of 3 different model families: logistic regression (LR), multi-layer perceptron (MLP), and random forest (RF) using implementations from sklearn~\citep{pedregosaScikitlearnMachineLearning2011}.
For each model family $m$ in LR, RF, and MLP, we train a set $\mathcal{H}_m$ of different candidates across a spectrum of hyperparameters designed to span under-fitting and over-fitting, hopefully including some well-fitting model instances.  

The full hyperparameter grid is provided below in Table \ref{tab:sup_hyper}. The logistic regression model was allowed a larger hyperparameter grid to compensate for its reduced parameter size, attempting to give each model family roughly equal computation runtime. 

\paragraph{Hardware.} All experiments were run using 4 CPU cores and 16GB of memory on a high-performance computing cluster. Only 30 minutes of wallclock time was allowed for each individual model fit. All runs completed within this limit.

% \begin{table*}[!t]
% \centering
% \begin{tabular}{l l l l l l l l l}
% \toprule
% Classifier & inverse regularization strength (C) & max\_iter & rare\_class\_weight & hidden\_layer\_sizes & alpha & learning\_rate\_init & max\_depth & min\_samples\_leaf \\
% \midrule
% Logistic regression & 10, 100, 1000, 10000, 100000 & 1000, 49, 7, 1 & 1, 3, 9, 27 & -- & -- & -- & -- & -- \\
% Multilayer perceptron (MLP) & -- & 100, 200 & -- & (64,), (64,64), (128,64) & 0.0001, 0.001 & 0.001, 0.0005 & -- & -- \\
% Random forest & -- & -- & 1, 3, 9, 27 & -- & -- & -- & 4, 16, 64 & 4, 16, 64 \\
% \bottomrule
% \end{tabular}
% \caption{Hyperparameter grid used in workflows/krh\_202509/step4\_run\_grid.sh (random seed omitted). C = inverse regularization strength.}
% \label{tab:sup_hyper}
% \end{table*}

\begin{table*}[ht]
\centering
\begin{tabular}{l | ccc}
\toprule
Hyperparameter & Logistic Regression & MLP & Random Forest \\
\midrule
Rare Class Weight & 1, 3, 9, 27 & -- & 1, 3, 9, 27 \\
Max Iterations & 1000, 49, 7, 1 & 100, 200 & -- \\
Inverse L2 Reg. Strength $C$ & 10, 100, 1000, 10000, 100000 & -- & -- \\
L2 Regularization Strength & -- & 0.0001, 0.001 & -- \\
Hidden Layer Sizes & -- & (64,), (64,64), (128,64) & -- \\
Learning Rate & -- & 0.001, 0.0005 & -- \\
Max Depth & -- & -- & 4, 16, 64 \\
Minimum Examples in Leaf & -- & -- & 4, 16, 64 \\
\bottomrule
\end{tabular}
\caption{Hyperparameter grid used in model selection experiments on MIMIC-IV and eICU.}
\label{tab:sup_hyper}
\end{table*}

\subsection{Procedure for model selection and deployment-aware testing}

For each of Scenario 1 and Scenario 2 described in the main paper, we define at the outset some desired constraints via specific values of $\alpha,\kappa$, as well as a desired non-uniform density over the fractional cost-parameter $p(t)$, via a chosen scenario-specific distribution over the cost ratio $\cRatio$.

We then perform the model selection and threshold selection phases described below for each scenario and dataset (MIMIC-IV and eICU). Each time, we begin with all ROC curves from the full hyperparameter grid described above in Tab.~\ref{tab:sup_hyper}.

%that is \emph{potentially non-uniform}, but that satisfies $\int_{t=a}^b p(t) dt = 1$ over a provided range $[a,b] \subseteq [0,1]$. For example, we can define a desired distribution over the cost ratio $C_0/C_1$ as in main paper, and map this to a density $p(t)$ either explicitly (via change of variables) or implicitly (via sampling cost ratios and mapping each sample to $t$). %Note that while our main paper defined partial VOROS as an integral over a uniform $t$ density, both VOROS and partial VOROS can account for a non-uniform $p(t)$ within the integral naturally, as described in \citet{ratigan2024voros}.

\textbf{Model selection phase.}
Denote the union of all trained model-hyperparameter configurations from  Tab.~\ref{tab:sup_hyper} as $\mathcal{H} = \mathcal{H}_{LR} \bigcup \mathcal{H}_{RF} \bigcup \mathcal{H}_{MLP}$.
Each element in $\mathcal{H}$ results in a score-producing classifier (as defined in Sec.~\ref{sec:background}) with its own ROC curve on the validation set. 

Given possible model-hyperparameter configurations $\mathcal{H}$ and their corresponding validation-set ROC curves, and a scenario-specific values for the constraints $\alpha,\kappa$ and a cost distribution $p(t)$, we use each of the following selection strategies
to pick a single model-hyperparameter combination $\mathsf{h}^* \in \mathcal{H}$:
\begin{itemize}
    \item maximizing partial VOROS: accounts for cost distribution $p(t)$ and constraints $\alpha,\kappa$
    \item maximizing total VOROS: accounts for cost distribution $p(t)$ but \emph{ignores} $\alpha,\kappa$
    \item maximizing recall in feasible region: accounts for constraints $\alpha,\kappa$ but ignores costs $p(t)$
    \item maximizing partial AUROC in feasible region: accounts for constraints $\alpha,\kappa$ but ignores costs $p(t)$
\end{itemize}
\textbf{Threshold selection phase.} Using the selected model-hyperparameter $\mathsf{h}^*$, we determine the binarization threshold $\tau(t)$ as follows:
\begin{itemize}
\item For cost-unaware strategies (max recall or max pAUROC): We use the same threshold regardless of $t$, so $\tau(t)$ is just a constant function. We pick the threshold value using $\mathsf{h}^*$'s validation set ROC curve, 
searching over all valid thresholds for the one that produces maximum recall while ensuring the predictions follow the precision and capacity constraints $\alpha,\kappa$.
\item For cost-aware strategies (partial VOROS or total VOROS): At each value of cost parameter $t$, we find a cost-specific threshold $\tau(t)$ 
using $\mathsf{h}^*$'s validation set ROC curve.
We search over all valid thresholds for the ROC point $h,k$ that both meets the desired precision and capacity constraints set by $\alpha,\kappa$ and also 
minimizes cost with respect to $t$ (using Defn.~\ref{defn:cost}).
\end{itemize}
% Using this $h^*$, we can revisit each $t$ value in the given range, and determine a specific binarization threshold $\tau(t)$ that performs best at that $t$. Some cost-unaware strategies like maximizing recall will use the same $\tau$ always, so $\tau(t)$ is just a flat function.

%\clearpage 

\textbf{Test phase.}
We wish to mimic authentic deployment of a real alert system for given limits of $\alpha, \kappa$ and costs $t \in [a,b]$ weighted by the given density $p(t)$.
We will force each selected model on the test data to use its predetermined threshold $\tau(t)$ to perform alerts, thus producing binary predictions (not scores) on the test set for a given $t$. This avoids ``tuning'' the selected threshold on the test set, which is not possible in prospective deployments.

On the test set using each predetermined threshold $\tau(t)$, we record the fpr,tpr location in ROC space as $(h_{\tau(t)},k_{\tau(t)})$. 
We then report the expected cost over the provided range:
\begin{align}
    \mathbb{E}_{t \sim p(t)}[ \text{Cost} ] &= \int_{t=a}^b p(t) \text{Cost}_t(h_{\tau(t)},k_{\tau(t)}) dt
    \\ \notag 
    &= \int_{t=a}^b p(t) \Big[ t h_{\tau(t)} + (1-t) (1- k_{\tau(t)}) \Big] dt
\end{align}
In this paper, throughout Sec.~\ref{sec:edelson_experiments}-\ref{sec:mimiciv_experiments} we always assume a uniform distribution over cost ratios $C_{\text{FP:FN}}$, between minimum ratio $a_c$ and maximum ratio $b_c$. 
We can then write the original expected cost as an integral over possible values $c$ for this ratio:
\begin{align}
    \mathbb{E}_{c \sim \text{Unif}(a_c,b_c)}[ \text{Cost} ] &= \int_{c=a_c}^{b_c} \frac{1}{b_c - a_c} \text{Cost}_{t(c)}(h_{\tau(c)},k_{\tau(c)}) dc
\end{align}

There are two ways to evaluate this integral in practice:
\begin{itemize}
\item Monte Carlo estimation. We can sample $S$ iid ratios $c_s \sim \text{Unif}(a_c, b_c)$, map each to a fractional cost parameter $t_s = t(c_s)$, then compute
\begin{align}
    &\approx \frac{1}{S} \sum_{s=1}^S \Big[ t_s h_{\tau(t_s)} + (1-t_s) (1- k_{\tau(t_s)}) \Big]
\end{align}

\item Numerical integration via the trapezoid method, with a sufficiently dense grid of $c$ values.
\end{itemize}
For all expected cost evaluations in this paper, we use the trapezoid method with a grid of 1000 evenly-spaced $c$ values from $a_c$ to $b_c$.

\clearpage 

\subsection{Results on MIMIC-IV and eICU: Comparing Model Selection Strategies}
\label{supp:results_eicu}
\label{supp:results_mimiciv}

Expanded tables showing many relevant performance metrics on both validation set and test set are in Tab.~\ref{tab:mimiciv_scenario_tables_detailed} and Tab.~\ref{tab:eicu_scenario_tables_detailed}.

\begin{figure}[!h]
    \centering
{\large Scenario 1 on MIMIC-IV} \medskip
    \resizebox{0.96\textwidth}{!}{%    
    \begin{tabular}{r r r r l | r r r l}
        & \multicolumn{8}{c}{
        $\alpha=0.15,
        ~~\frac{C_0}{C_1}\sim \text{Unif}(\frac{1}{9},\frac{1}{6})$
        } \\
        \hline
            & \multicolumn{4}{c}{on Val. set: $t \in [0.5,0.6], \kappa = 0.5 |\mathcal{D}| = 3901$,}
            &  \multicolumn{4}{c}{on Test set: $t \in [0.48,0.58], \kappa = 0.5 |\mathcal{D}| = 3930.5$}
        \\
        Strategy 
            & Avg.~Cost & Precision & Capacity & Satisfied?
            & Avg.~Cost & Precision & Capacity & Satisfied? \\
        \hline
         max VOROS 
            & \textbf{0.266} & 0.250 & 2277.5 & Yes
            & \textbf{0.261} & 0.262 & 2310.1 & Yes\\
         max pAUROC 
            & 0.313 & 0.178 & 3896 & Yes
            & 0.303 & 0.188 & 3919 & Yes \\
         max recall 
            & 0.304 & 0.181 & 3888 & Yes
            & 0.302 & 0.188 & 3926 & Yes\\
         max PV (ours) 
         & \textbf{0.266} & 0.250 & 2277.5 & Yes
         & \textbf{0.261} & 0.262 & 2310.1 & Yes
    \end{tabular}
    }% end resizebox
    \\
    \medskip  
       
    {\large Scenario 2 on MIMIC-IV} \medskip
    \resizebox{0.96\textwidth}{!}{%
    \begin{tabular}{r r r r l | r r r l}
        & \multicolumn{8}{c}{
        $\alpha=0.5$,
        
        ~~$\frac{C_0}{C_1}\sim \text{Unif}(\frac{1}{40},\frac{1}{20})$
        } \\
        \hline
            & \multicolumn{4}{c}{on Val. set: $t \in [0.18,0.31], \kappa=0.1 |\mathcal{D}| = 780.2$,}
            & \multicolumn{4}{c}{on Test set: $t \in [0.17,0.29], \kappa=0.1 |\mathcal{D}| = 786.1$}
        \\
        Strategy 
            & Avg.~Cost & Precision & Capacity & Satisfied?
            & Avg.~Cost & Precision & Capacity & Satisfied? \\
        \hline
         max VOROS 
            & 0.665 & 0.505 & 196 & Yes
            & 0.640 & 0.597 & 236 & Yes\\
         max pAUROC 
            & \textbf{0.561} & 0.501 & 429 & Yes
            & \textbf{0.538} & 0.533 & 486 & Yes\\
         max recall 
            & \textbf{0.561} & 0.501 & 429 & Yes
            &  \textbf{0.538} & 0.533 & 486 & Yes\\
         max PV (ours) 
         & \textbf{0.561} & 0.501 & 429 & Yes
         & \textbf{0.538} & 0.533 & 486 & Yes
    \end{tabular}
    }
    \caption{\textbf{Detailed results for selected classifiers on MIMIC-IV.}
    Average cost, precision, capacity on valid and test sets across two misclassification cost scenarios for maximum (full) VOROS, maximum partial AUROC, maximum recall, and maximum partial VOROS strategies. Here on MIMIC-IV, no selected models violated constraints on test.
    }%endcaption
    \label{tab:mimiciv_scenario_tables_detailed}
\end{figure}

\begin{figure}[!h]
    \centering
    % & scenario 1 & scenario 2 \\
    %     & $\alpha=0.15,\kappa=0.5$ & $\alpha=0.5,\kappa=0.1$\\
    %     & $\frac{C_0}{C_1}\sim \text{Unif}(\frac{1}{9},\frac{1}{6})$ & $\frac{C_0}{C_1}\sim \text{Unif}(\frac{1}{40},\frac{1}{20})$\\
    %     & $t \in [0.46,0.56]$ & $t \in [0.16,0.28]$\\
    {\large Scenario 1 on eICU} \medskip
    \resizebox{0.96\textwidth}{!}{%
    \begin{tabular}{r r r r l | r r r l}
        & \multicolumn{8}{c}{
        $\alpha=0.15,
        ~~\kappa = 0.5 |\mathcal{D}| = 9633.5$,
        ~~$\frac{C_0}{C_1}\sim \text{Unif}(\frac{1}{9},\frac{1}{6})$
        } \\
        \hline
            & \multicolumn{4}{c}{on Val. set: $t \in [0.46,0.56]$}
            & \multicolumn{4}{c}{on Test set: $t \in [0.46,0.56]$}
        \\
        Strategy 
            & Avg.~Cost & Precision & Capacity & Satisfied?
            & Avg.~Cost & Precision & Capacity & Satisfied? \\
        \hline
         max VOROS 
            & \textbf{0.307} & 0.224 & 6845.4 & Yes
            & \textbf{0.319} & 0.215 & 7037.7 & Yes\\
         max pAUROC 
            & 0.325 & 0.186 & 9631 & Yes
            & 0.336 & 0.182 & {\color{red} 9800} & No (capacity) \\
         max recall 
            & 0.325 & 0.186 & 9631 & Yes
            & 0.336 & 0.182 & {\color{red} 9800} & No (capacity)\\
         max PV (ours) 
         & \textbf{0.307} & 0.224 & 6845.4 & Yes
         & \textbf{0.319} & 0.215 & 7037.7 & Yes
    \end{tabular}
    }%end resizebox
    \\
    \medskip  
       
    {\large Scenario 2 on eICU} \medskip
    \resizebox{0.96\textwidth}{!}{%
    \begin{tabular}{r r r r l | r r r l}
        & \multicolumn{8}{c}{
        $\alpha=0.5,
        ~~\kappa=0.1 |\mathcal{D}| = 1926.7$,
        ~~$\frac{C_0}{C_1}\sim \text{Unif}(\frac{1}{40},\frac{1}{20})$
        } \\
        \hline
            & \multicolumn{4}{c}{on Val. set: $t \in [0.16,0.28]$}
            & \multicolumn{4}{c}{on Test set: $t \in [0.16,0.28]$}
        \\
        Strategy 
            & Avg.~Cost & Precision & Capacity & Satisfied?
            & Avg.~Cost & Precision & Capacity & Satisfied? \\
        \hline
         max VOROS 
            & 0.773 & 0.500 & 12 & Yes
            & \textbf{0.772} & 0.684 & 19 & Yes\\
         max pAUROC 
            & \textbf{0.693} & 0.501 & 483 & Yes
            & 0.707 & {\color{red} 0.437} & 474 & No (precision)\\
         max recall 
            & \textbf{0.693} & 0.501 & 483 & Yes
            &  0.707 &  {\color{red} 0.437} & 474 & No (precision)\\
         max PV (ours) 
         & \textbf{0.693} & 0.501 & 483 & Yes
         & 0.707 &  {\color{red} 0.437} & 474 & No (precision)
    \end{tabular}
    }%end resizebox
    \\
        \caption{\textbf{Detailed results for selected classifiers on eICU.}
        Average cost, precision, capacity on valid and test sets across two misclassification cost scenarios for maximum (full) VOROS, maximum partial AUROC, maximum recall, and maximum partial VOROS strategies. Results that violate constraints are highlighted red.}
    \label{tab:eicu_scenario_tables_detailed}
\end{figure}

\newpage

\setlength{\tabcolsep}{1mm}
\begin{figure*}
\centering
\begin{tabular}{c c}
\includegraphics[height=5cm,width=0.45\textwidth,keepaspectratio]{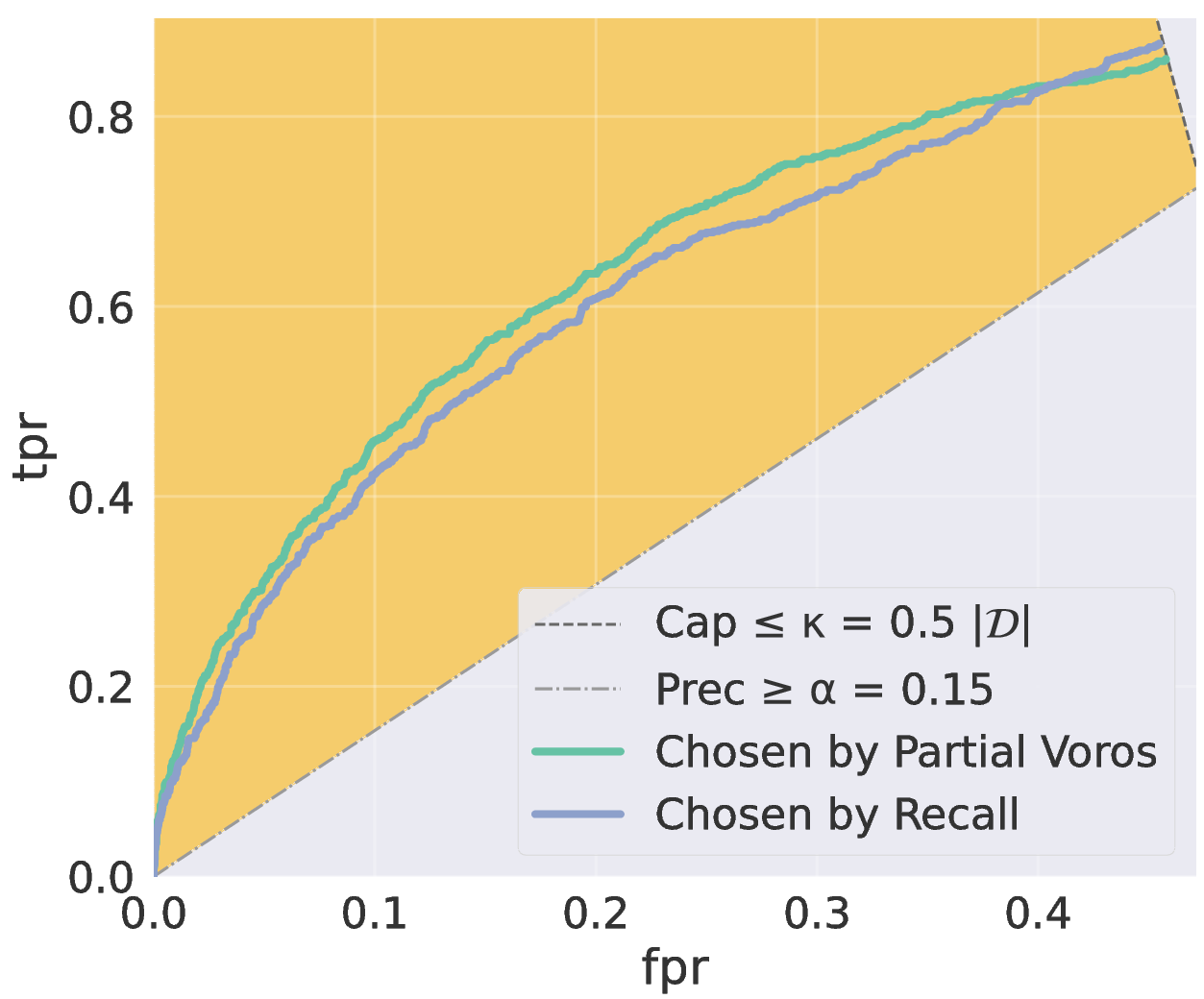}
&
\includegraphics[height=5cm,width=0.45\textwidth,keepaspectratio]{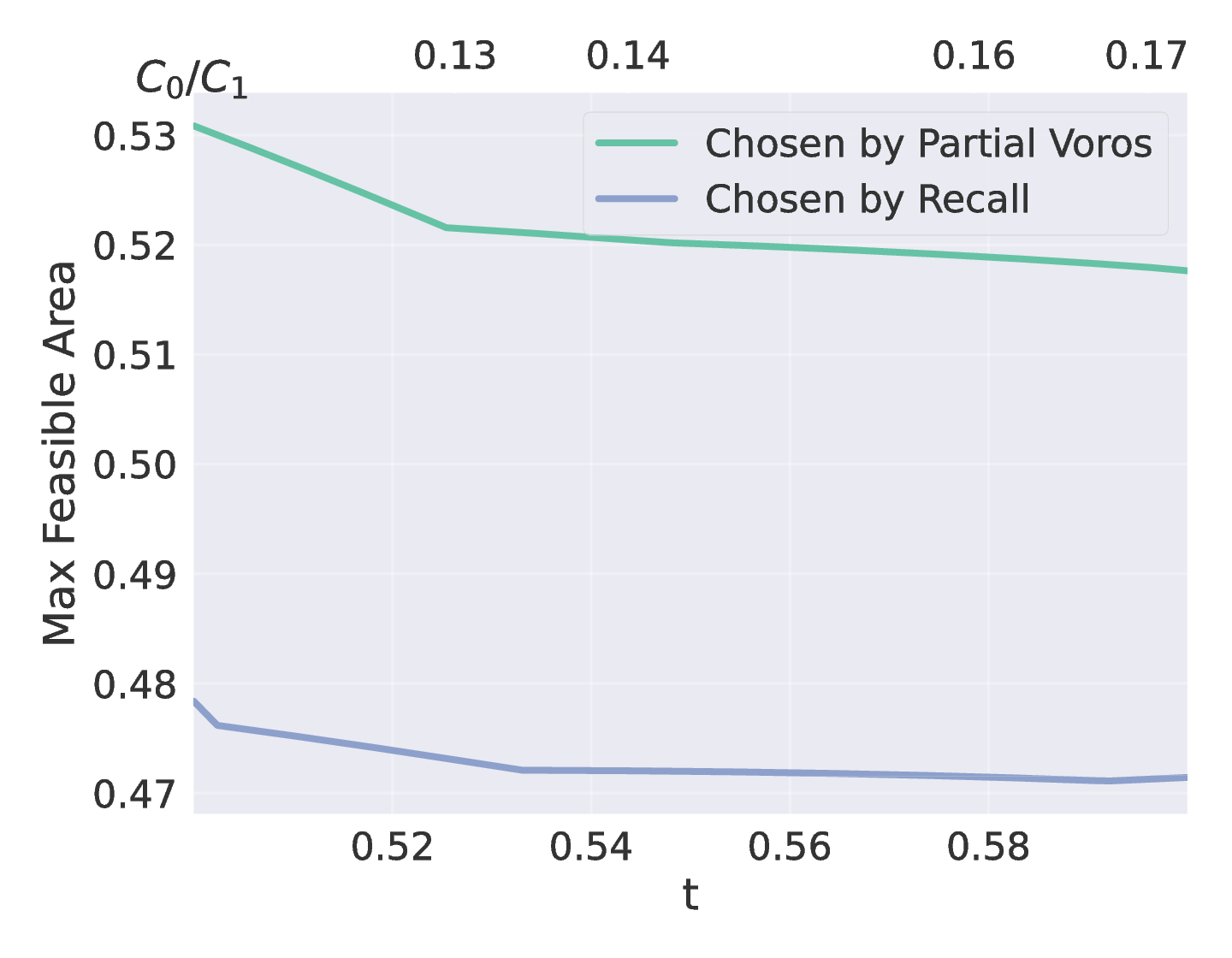}
% &
% \begin{minipage}[t]{0.42\textwidth}
%   \vspace{2mm} % NEEDED TO ALIGN TO TOP
% \resizebox{\textwidth}{!}{%
% \begin{tabular}{l r r}
% & scenario 1 & scenario 2
% \\
% & $\alpha = 0.15, \kappa = 0.5 |\mathcal{D}|$
% & $\alpha = 0.5, \kappa = 0.1 |\mathcal{D}|$
% \\
% & $\frac{\cFP}{\cFN} \sim \text{Unif}(\frac{1}{9}, \frac{1}{6} )$
% & $\frac{\cFP}{\cFN} \sim \text{Unif}(\frac{1}{40}, \frac{1}{20} )$
% \\
% & $t \in [.5,.6]$
% & $t \in [.18,.31]$
% \\
% \hline
% Strategy 
%  & Avg. Cost & Avg. Cost
% \\
% \hline
% max VOROS
% & 0.306
% & 0.636
% \\
% max pAUROC
% & 0.306
% & \textbf{0.535}
% \\
% max recall 
% & 0.305
% & \textbf{0.535}
% \\
% max PV (ours)
% & \textbf{0.261}
% & \textbf{0.535}
% \end{tabular}
% }%end resizebox
% \end{minipage}
\end{tabular}
\caption{
\textbf{Diagnostic visuals for model selection on MIMIC-IV.}
We compare our \emph{partial VOROS} against other strategies for selecting model family, hyperparameters, and  decision thresholds using validation data. We evaluate each selection's binary alerts on test data in terms of average cost over the provided distribution for cost ratio $\cRatio$.
\emph{Left:} The ROC curve on heldout data in the feasible region (gold), for the two best  selection strategies in scenario 1. 
%$\alpha=0.15, \kappa=0.5 |\mathcal{D}|, t \in [0.5, 0.6]$
\emph{Right:} The maximum feasible area of lesser classifiers across a range of cost parameter values $t$, with the corresponding cost ratio $\cRatio$ on the top axis.
%$\alpha=0.15, \kappa=0.5 |\mathcal{D}|, t \in [0.5, 0.6]$
% \emph{Right:} Table of costs on test set for each strategy and scenario.
}%endcaption
\label{suppfig:results_mimiciv_2panel}
\end{figure*}

\begin{figure}[h!]
    \centering
    \begin{tabular}{c c}
    \includegraphics[width=0.45\textwidth,height=5cm,keepaspectratio]{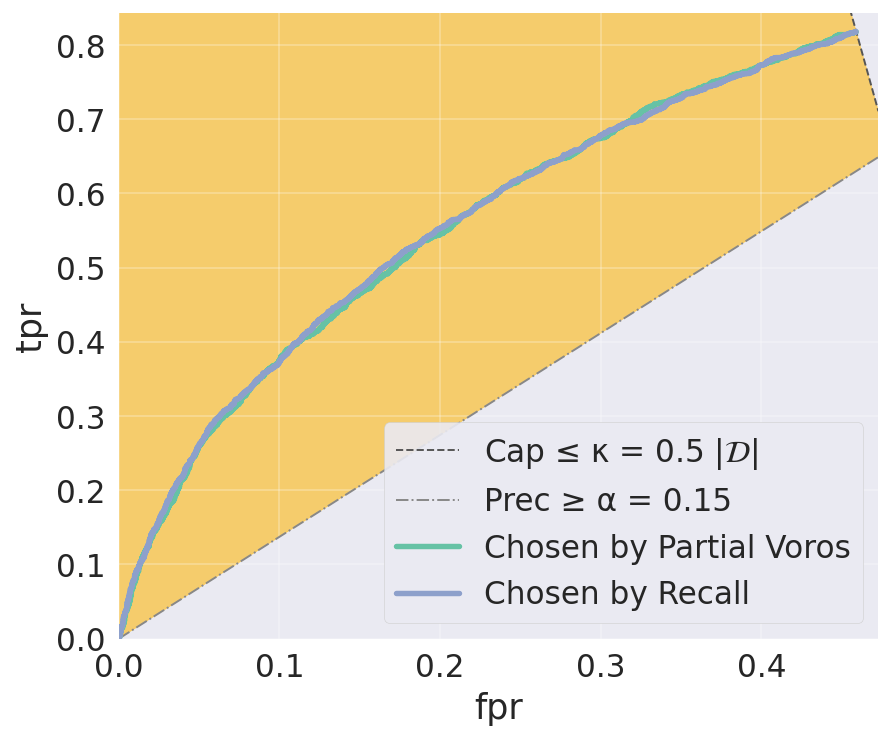}
    &
    \includegraphics[width=0.45\textwidth,height=5cm,keepaspectratio]{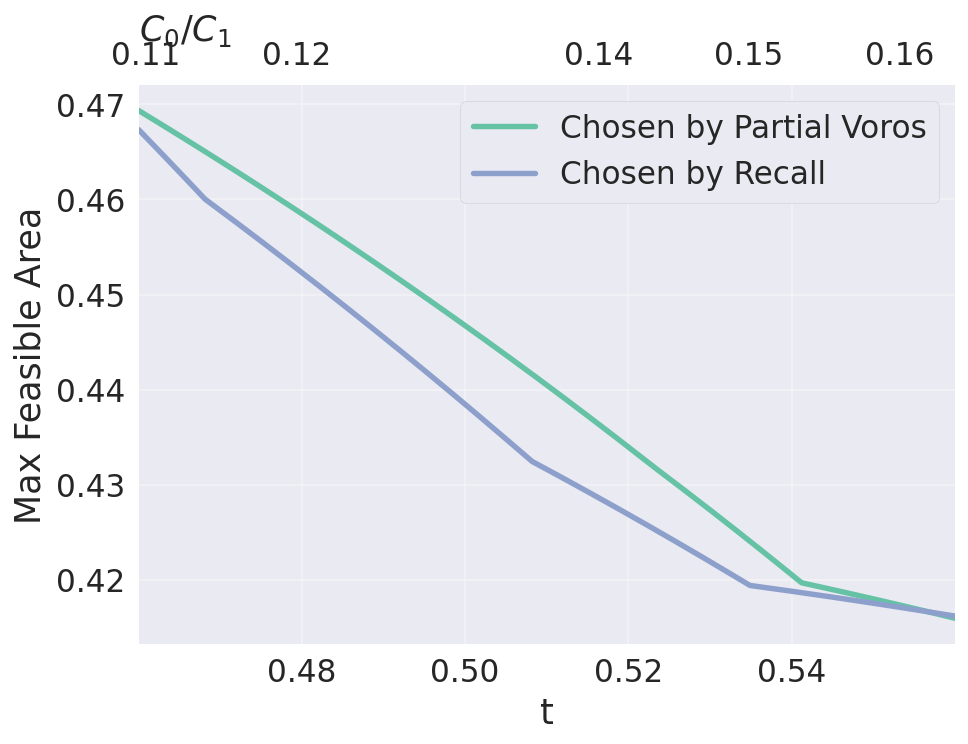}
    \end{tabular}
    \caption{\textbf{Diagnostic visuals for model selection on eICU.} We show plots for models selected by partial VOROS and recall strategies for Scenario 1.
    \emph{Left:} ROC curves in the feasible region on the validation set.
    \emph{Right:} Maximum feasible area of lesser classifiers for a range of cost parameter $t \in [0.46, 0.56]$ on the validation set.
    }%end caption
    % \caption{Maximum feasible area of lesser classifiers for a range of cost parameter $t \in [0.46, 0.56]$.}
    \label{suppfig:eicu_curves}
\end{figure}

\end{document}